\documentclass[11pt]{article}

\usepackage{microtype}
\usepackage{graphicx}
\usepackage{subfigure}
\usepackage{booktabs} 

\usepackage{hyperref}       
\usepackage{url}            
\usepackage{booktabs}       
\usepackage{amsfonts}       
\usepackage{microtype}      
\usepackage{graphicx}

\usepackage{fullpage}
\usepackage[round]{natbib} 
\usepackage{color}
\usepackage{hyperref}

\usepackage[margin=1in]{geometry}

\usepackage{lipsum}  
\makeatletter
\newcommand{\printfnsymbol}[1]{%
  \textsuperscript{\@fnsymbol{#1}}%
}
\makeatother

\usepackage{qiangstyle}

\title{\huge Doubly Robust Bias Reduction in 
Infinite Horizon Off-Policy Estimation}
\author{%
	Ziyang Tang\thanks{The first two authors contributed equally to this work.}\\
    ~University of Texas at Austin\\
	\texttt{ztang@cs.utexas.edu}
	\and
	Yihao Feng\printfnsymbol{1}\\
	~University of Texas at Austin\\
	\texttt{yihao@cs.utexas.edu} \\
	\and
	Lihong Li\\
	Google Research\\
	\texttt{lihong@google.com}\\
	\and
    Dengyong Zhou\\
	Google Research\\
	\texttt{dennyzhou@google.com}\\
	\and 
	Qiang Liu\\
	University of Texas at Austin\\
	\texttt{lqiang@cs.utexas.edu} \\
}

\date{}

\begin{document}
\maketitle

\begin{abstract}
  \emph{Infinite horizon} off-policy policy evaluation is a highly challenging task due to the excessively large variance of typical importance sampling (IS) estimators. 
Recently, \citet{liu2018breaking} proposed an approach 
that significantly reduces the variance of 
infinite-horizon off-policy evaluation by estimating the stationary density ratio, but at the cost of introducing potentially high biases  due to the error in density ratio estimation. 
In this paper, we develop a bias-reduced augmentation of their method, which can take advantage of a learned value function to obtain higher accuracy. 
Our method is doubly robust in that the bias vanishes when either the density ratio or the value function estimation is perfect.  In general, when either of them is accurate, the bias can also be reduced.
Both theoretical and empirical results show that our method yields significant advantages over previous methods.
\end{abstract}

\section{Introduction}

A key problem in reinforcement learning (RL)~\citep{sutton98beinforcement} is off-policy policy evaluation: given a fixed \emph{target policy} of interest, 
estimating the average reward garnered by an agent that follows the policy, by only using data collected from  different \emph{behavior policies}. 
This 
problem is widely encountered in many real-life applications~\citep[e.g.,][]{murphy01marginal,li11unbiased,bottou13counterfactual,thomas17predictive}, where online experiments are expensive and high-quality simulators are difficult to build. 
It also serves as a key algorithmic component of off-policy policy optimization~\citep[e.g.,][]{dudik11doubly,jiang16doubly,thomas16data, liu2019off}.

There are  two major families of approaches for policy evaluation.  The first approach is to build a simulator that mimics the reward and next-state transitions of the real environment~\citep[e.g.,][]{fonteneau13batch}.  While straightforward, this approach strongly relies on the model assumptions in building the simulator, which may invalidate evaluation results.  The second approach is to use importance sampling to correct the sampling bias in off-policy data, so that an (almost) unbiased estimator can be obtained~\citep{liu01monte,strehl11learning,bottou13counterfactual}.  A major limitation, however, is that importance sampling can become inaccurate due to high variance. 
In particular, 
most existing IS-based estimators compute the weight as the product of the importance ratios of many steps in the trajectory, 
causing excessively high variance for problems with long or infinite horizon, yielding a \emph{curse of horizon} 
\citep{liu2018breaking}.  
Recently, \citet{liu2018breaking} proposes a new estimator for infinite-horizon off-policy evaluation, which presents significant advantages to standard importance sampling methods.   
Their method directly estimates the density ratio between 
the state stationary distributions 
of the target and behavior policies, 
instead of the trajectories, thus avoiding exponential blowup of  variance in the horizon.  
%
%
%
While \citeauthor{liu2018breaking}'s 
method shows much promise by significantly \emph{reducing the variance}, in practice, it may suffer from \textit{high bias}  due to the error or model misspecficiation when estimating the density ratio function.  

In this paper, we develop a \emph{doubly robust} estimator 
for infinite horizon off-policy estimation, 
by integrating \citeauthor{liu2018breaking}'s 
method  with information from an additional  
value function estimation.  
This significantly reduces the bias of 
\citeauthor{liu2018breaking}'s method 
once either the density ratio, or the value function estimation is accurate (hence doubly robust). Since 
\citeauthor{liu2018breaking}'s 
method already promises low variance, our additional bias reduction allows us to achieve significantly better accuracy for practical problems.

Technically, 
our new \emph{bias reduction} method 
provides a new angle of double robustness 
for off-policy evaluation, 
orthogonal to the existing literature of 
doubly robust policy evaluation that solely devotes to \emph{variance reduction} \citep{jiang16doubly,thomas16data,farajtabar18more}, mostly based on the idea of control variates \citep[e.g.][]{asmussen2007stochastic}.
Our double robustness for bias reduction 
is significantly different, and instead yields an intriguing connection with the fundamental primal-dual relations between the density (ratio) functions and value functions \citep[e.g.,][]{bertsekas1995dynamic, puterman2014markov}. 
Our new perspective may allow us to motivate new algorithms for more efficient policy evaluation, and develop unified frameworks for combining these two types of double robustness in future work.

\section{Background}

\paragraph{Infinite Horizon Off-Policy Estimation} Let 
$M=\langle\Sset, \Aset, r, \T, \mu_0 \rangle$ be a Markov decision process (MDP) with state space $\Sset$, action space $\Aset$, reward function $r$, transition probability function $\T$, and initial-state distribution $\mu_0$. 
A policy $\pi$ maps states to distributions over $\Aset$, with $\pi(a|s)$ being the probability of selecting $a$ given $s$.
The average discounted reward for $\pi$, with a given discount $\gamma\in(0,1)$ 
\footnote{For average case when $\gamma = 1$, the definition of $\Rpi$ is the same.
However, the definition of value function is different.
We will assume $\gamma < 1$ throughout our main paper for simplicity;
for average case check appendix \ref{sec:avg_prime} for more details.}, is defined as
\begin{equation*}
    \Rpi := \lim_{T\to \infty} \E_{\tau \sim \pi}\left [\frac{\sum_{t=0}^T \gamma^t r_t}{\sum_{t=0}^T \gamma^t} \right],
\end{equation*}
where $\tau = \{s_t,a_t,r_t\}_{0\leq t\leq T}$ is a trajectory with states, actions, and rewards collected by following policy $\pi$ in the MDP, starting from $s_0 \sim \mu_0$.
Given a set of $n$ trajectories, $\D = \{s_t^{(i)},a_t^{(i)},r_t^{(i)}\}_{1\leq i\leq n, 0\leq t\leq T}$, collected under a behavior policy $\pi_0(a|s)$, the off-policy evaluation problem aims to estimate the average discounted reward $\Rpi$ for another target policy $\pi(a|s)$.


\paragraph{Estimation via Value Function} 
The value function for policy $\pi$ is defined as the expected accumulated discounted future rewards started from a certain state: $V^\pi(s) = \E_{\tau\sim \pi}[\sum_{t=0}^{\infty}\gamma^t r_t|s_0=s]$.
We use $r^\pi(s) = \E_{a\sim \pi(\cdot|s)}[r(s,a)]$ to denote the average reward for state $s$ given policy $\pi$.
Under the definition, the value function can be seen as a fixed point of the Bellman equation:
\begin{align} 
      V^{\pi}(s) = r^\pi(s) + \gamma \Ppi V^\pi (s),
      &&
      \Ppi V^\pi (s) := 
      \E_{a\sim \pi(\cdot|s), s'\sim \T(\cdot|s,a)}[V^\pi(s')], 
      ~~~~\forall s\in \Sset, 
    \label{eqn:bellman_V}
\end{align}
where $\Ppi V(s)$ is the average of the next value function given the current state $s$ and policy $\pi$ (check appendix \ref{section:operator} for details).

The value function and the expected reward $\Rpi$ 
is related in the following straightforward way  
\begin{align}
    \label{equ:estimationValue}
    \Rpi = (1-\gamma) \E_{s\sim \mu_0}[V^\pi(s)],
    \end{align}
    where the expectation is with respect to the  distribution $\mu_0(s)$ of the initial states $s_0$ at time $t$. 
    Therefore, given an
approximation $\vhat$ of $V^\pi$, and samples $\mathcal D_0:= \{s_0^{(i)}\}_{1\leq i\leq n_0}$ drawn from $\mu_0(s)$, 
we can estimate $\Rpi$ by 
\begin{equation*}
    \Rval[\vhat] = \frac{(1-\gamma)}{n_0}\sum_{i=1}^{n_0}\vhat(s_0^{(i)}).
\end{equation*}
Note that this estimator is off-policy in nature, 
since it requires no samples from the target policy $\pi$. 

\paragraph{Estimation via State Density Function} 
Denote $d_{\pi,t}(\cdot)$ as average visitation of $s_t$ in time step $t$. The state density function, or the discounted average visitation, is defined as:
\begin{equation*}
    d_\pi(s) := \lim_{T\to \infty}\frac{\sum_{t=0}^{T} \gamma^t d_{\pi,t}(s)}{\sum_{t=0}^{T} \gamma^t} = (1-\gamma)\sum_{t=0}^{\infty} \gamma^t d_{\pi,t}(s),
\end{equation*}
where $(1-\gamma)$ can be viewed as the normalization factor introduced by $\sum_{t=0}^{\infty} \gamma^t$.

Similar to Bellman equation for value function, the state density function can also be viewed as a fixed point to the following recursive equation (\cite{liu2018breaking}, Lemma 3):
\begin{align} 
    d_\pi(s') = (1-\gamma)\mu_0(s') + 
    \gamma \mathcal T^\pi d_\pi(s'), 
    &&\text{where} &&
     \mathcal T^\pi d_\pi(s') \defeq \sum_{s,a}\T(s'|s,a)\pi(a|s) d_\pi(s). 
    \label{eqn:bellman_density}
\end{align}
The operator $\mathcal T^\pi$ is an adjoint operator of $\mathcal P^\pi$ used in \eqref{eqn:bellman_V}. See Appendix~\ref{section:operator} for discussion. 

%
If the density function $d_{\pi}$ is known, 
it provides an alternative way for estimating the expected reward $\Rpi$, by noting that 
\begin{align}
    \Rpi = \E_{s\sim d_{\pi}, a\sim \pi(\cdot|s)}[r(s,a)].
    \label{eqn:EstimationDensity}
\end{align}

We can see that both density function $d_\pi$ and value function $V^\pi$ can be used to estimate the expected reward $\Rpi$.
We clarify the connection in detail in Appendix \ref{section:operator}. 

\paragraph{Off-Policy State Visitation Importance Sampling}
Equation~\eqref{eqn:EstimationDensity} can not be directly
used for off-policy estimation, since it involves expectation under the behavior policy $\pi$. 
\citet{liu2018breaking} addressed this problem by introducing a change of measures via importance sampling: 
\begin{align}
    \Rpi 
    = \E_{s\sim d_{\pi_0}, a\sim \pi_0(\cdot|s)}\left [\rnd(s) \frac{\pi(a|s)}{\pi_0(a|s)} r(s,a) \right], 
    && \text{with} &&  \rnd(s) = \frac{d_\pi(s)}{d_{\pi_0}(s)}, 
\end{align}
where $\rnd(s)$  is 
the density ratio function of 
$d_\pi$ and $d_{\pi_0}$. 

Given an approximation $\what$ of $\rnd$, and samples $\D = \{s_t^{(i)},a_t^{(i)},r_t^{(i)}\}_{1\leq i\leq n, 0\leq t\leq T}$ collected from policy $\pi_0$, we can estimate $\Rpi$ as:
\begin{align} \label{eqn:SIS}
    \Rsis[\what] = \frac{1}{Z}\sum_{i=1}^n\sum_{t=0}^T \gamma^t \what(s_t^{(i)}) \frac{\pi(a_i|s_i)}{\pi_0(a_i|s_i)} r_i,
    &&
Z = \sum_{i=1}^n\sum_{t=0}^T \gamma^t \what(s_t^{(i)})
\frac{\pi(a_t^{(i)}|s_t^{(i)})}{\pi_0(a_t^{(i)}|s_t^{(i)})}, 
\end{align} 
where $Z$ is the normalized constant of the importance weights. 

\section{Doubly Robust Estimator}
Doubly robust estimator is first proposed into reinforcement learning community to solve contextual bandit problem by \citet{dudik11doubly} as an estimator combining \textit{inverse propensity score} (IPS) estimator and \textit{direct method} (DM) estimator.

\citet{jiang16doubly} introduces the idea of doubly robust estimator into off-policy evaluation in reinforcement learning. 
It incorporates an approximate value function as a control variate to reduce the variance of importance sampling estimator.
Inspired by previous works, we propose a new doubly robust estimator based on our infinite horizon off policy estimator $\Rsis$.


\newcommand{\connecting}{\text{bridge~}}
\subsection{Doubly Robust Estimator for Infinite Horizon MDP}
The value-based estimator
$\Rval[\vhat]$ and 
density-ratio-based estimator $\Rsis[\what]$ are expected to be accurate 
when $\vhat$ and $\what$ are accurate, respectively.
Our goal is to combine their advantages, obtaining a doubly robust estimator that is accurate once either $\vhat$ or $\what$ or is accurate.

To simplify the problem, it is useful to exam 
the limit of infinite samples, 
with which $\Rval[\vhat]$ and 
$\Rsis[\what]$ converge to the following limit of expectations: 
\begin{align}
    &\Rsisi[\what] := \lim_{n,T\to\infty} \Rsis[\what] = 
    \sum_{s} r^\pi(s) d_{\pi_0}(s) \what(s), \label{eqn:Rsisi}\\
    & \Rvali[\vhat] := \lim_{n_0\to\infty} \Rval[\vhat]  =  
    (1-\gamma) \sum_{s} \vhat(s) \mu_0(s). \label{eqn:Rvali}
\end{align}
Here and throughout this work, we assume $\vhat$ and $\what$ are fixed pre-defined approximations, and only consider the randomness from the data $\mathcal D$. 

A first observation is that we expect to have $r^\pi \approx  \vhat - \gamma \Ppi \vhat$ by Bellman equation~\eqref{eqn:bellman_V}, when $\hat V$ approximates the true value $V^\pi$.  Plugging this into $\Rsisi[\what]$ in Equation~\eqref{eqn:Rsisi}, we obtain the following ``\connecting estimator'', 
which incorporates information from both $\what$ and $\vhat$. 
\begin{equation}
    \Rconni[\vhat, \what] = \sum_s \left(\vhat(s) - \gamma \Ppi \vhat(s)\right) d_{\pi_0}(s) \what(s),
\end{equation}
where operator $\Ppi$ is defined in Bellman equation~\eqref{eqn:bellman_V}. 
The corresponding empirical estimator is defined by 
\begin{equation} \label{eqn:empiricalConn}
    \Rconn[\vhat, \what]  =\sum_{i=1}^n \sum_{t=0}^{T-1}
    \left (
     \frac{1}{Z_1} \gamma^{t}\what(s_t^{(i)})\vhat(s_t^{(i)}) - \frac{1}{Z_2}
    \gamma^{t+1} \what(s_t^{(i)}) 
    \frac{\pi(a_i|s_i)}{\pi_0(a_i|s_i)}
    \vhat(s_{t+1}^{(i)})
    \right),
\end{equation}
where $Z_1 = \sum_{i=1}^n \sum_{t=0}^{T-1} \gamma^{t}\what(s_t^{(i)})$ and $Z_2 = \sum_{i=1}^n \sum_{t=0}^{T-1}\gamma^{t+1} \what(s_t^{(i)}) \betar(a_t^{(i)}|s_t^{(i)})$ are self-normalized constant of important weights each empirical estimation. 

However, directly estimating $R^\pi$ using the bridge estimator $\Rconn[\vhat, \what]$ yields a poor estimation,
because it includes the errors from both $\what$ and $\vhat$ and is in some sense ``\textit{doubly worse}''. 
However, we can construct our 
``\textit{doubly robust}'' estimator by 
``canceling $\Rconn[\vhat, \what]$ out from $\Rsisi[\what]$ and $\Rvali[\vhat]$''. 
    \begin{equation}
        \Rdr[\vhat, \what] = \Rsis[\what] + \Rval[\vhat] - \Rconn[\vhat, \what]. 
        \label{eqn:doubly_robust}
    \end{equation}

\paragraph{Doubly Robust Bias Reduction} 
The double robustness of $\Rdr[\vhat, \what]$ is reflected in the following key theorem, which shows that it is accurate once either $\vhat$ or $\what$ is accurate. 
\begin{thm}[Double Robustness]
Let $\Rdri[\vhat, \what] := \lim_{n_0,n,T\to \infty}\Rdr[\vhat, \what]$ be the limit of $\Rdr$ when it has infinite samples.
Following the definition above, we have 
\begin{equation}
    \Rdri[\vhat, \what]  - \Rpi
    = 
    \E_{s\sim d_{\pi_0}} \left [\varepsilon_{\what}(s)\varepsilon_{\vhat}(s)\right],  
\end{equation}
where $\varepsilon_{\vhat}$ and $\varepsilon_{\what}$ are errors of $\vhat$ and $\what$, respective, defined as follows 
\begin{align*} 
\varepsilon_{\what} =\frac{d_\pi(s)}{d_{\pi_0}(s)} - 
{\what}(s), &&
\varepsilon_{\vhat}(s) = \vhat(s) - r^\pi(s) - \gamma \Ppi \vhat(s). 
\end{align*}
The error $\varepsilon_{\what}$ of $\what$ is 
measured by the difference with the true density ratio $d_\pi(s)/d_{\pi_0}(s)$, and the error $\varepsilon_{\vhat}$ of $\vhat$ is measured using the Bellman residual. 
\label{thm:biased}
\end{thm}
If $\vhat$ is exact ($\vhat\equiv V^\pi$), we have $\varepsilon_{\vhat} \equiv 0$; 
if $\what$ is exact ($\what \equiv d_\pi/d_{\pi_0}$), we have $\varepsilon_{\what}\equiv 0$. 
Therefore, our estimator is consistent
(i.e., $\lim_{n,n_0\to \infty}\Rdr[\vhat, \what] = \Rpi$) if either $\vhat$ or $\what$ are exact. In comparison, $\Rsis[\what]$ and $\Rval[\vhat]$ are sensitive to the error of $\what$ and $\vhat$, respectively. We have 
\begin{align*}
    \Rsisi[\what] - \Rpi = \E_{s\sim d_{\pi_0}} \left [\varepsilon_{\what}(s)r^\pi(s)\right], && \Rvali[\vhat] - \Rpi = \E_{s\sim d_{\pi_0}} \left [\rnd(s) \varepsilon_{\vhat}(s)\right].
\end{align*}

%


\paragraph{Variance Analysis} 
Different from the bias reduction, 
the doubly robust estimator does not guarantee to the reduce the variance over $\Rsis[\what]$ or $\Rval[\vhat]$. However, as we show in the following result, the variance of $\Rdr[\hat V, \hat w]$ is controlled by $\Rsis[\what]$ and $\Rval[\vhat]$, both of which are already relatively small by the design of both methods. 
In addition, our method can have significant reduction of variance over $\Rsis[\what]$ when $\hat V \approx V$, which can have much larger variance than $\Rval[\vhat]$. 
\begin{thm}[Variance Analysis]
\label{thm:variance}
Assume $\Rdr[\vhat, \what]$ is estimated based sample $\mathcal D_0 \sim \mu_0$ and  $\mathcal D_{\pi_0} \sim d_{\pi_0}$, which we assume to be independent with each other. 
For simplicity, assume constant normalization is used in importance sampling (hence an unbiased estimator).
We have
\begin{equation}\label{eqn:varianceDR}
    \var_{\mathcal D_0,\mathcal D_{\pi_0}}\left[\Rdr[\vhat, \what]\right] = 
    \var_{\mathcal D_0}\left[\Rval[\vhat]\right] ~+~
    \var_{\mathcal D_{\pi_0}}\left [ \Rres[\vhat, \what] \right ], 
\end{equation}
with 
$$
\Rres[\vhat, \what] = \hat \E_{\mathcal D_{\pi_0} } \left[
\left(\hat r^\pi(s) - \left(\vhat(s) - \gamma \widehat{\Ppi} \vhat(s)\right) \right) 
\hat w(s) \right], 
$$
where $\hat r^\pi(s) = r(s,a)\pi(a|s)/\pi_0(a|s)$, $\widehat{\Ppi} \vhat(s) = \pi(a|s)/\pi_0(a|s) \vhat(s')$. 
In comparison, recall that $\Rsis[\what] = \hat \E_{{\mathcal D_{\pi_0}}}[\hat r^\pi(s) \what (s)].$  
Therefore, $\Rres[\vhat, \what]$ 
can have lower variance than $\Rsis[\what]$ when $\vhat$ is close to the true value $V^\pi$, or $\vhat - \Ppi\vhat \approx r^\pi$. 
\end{thm}
The theorem shows the variance of our doubly robust comes from two parts: the variance for value function estimation and a variance-reduced variant of $\Rsis$, when $\vhat \approx V^\pi$.
\eqref{eqn:varianceDR} shows that our variance is always larger than that $\Rval[\vhat]$, 
however, it can have lower variance than $\Rsis[\what]$, relevant to practice.
This is because the variance of $\Rval[\vhat]$ can be very small if we can draw a lot of samples from $\mu_0$, and 
$\Rsis[\what]$ may have larger variance if the variance of the density ratio $\what(s)$ and $\rnd(s)$ are large. 
Meanwhile, the variance of both $\Rval[\vhat]$ and $\Rsis[\what]$, by their design, are already much smaller than typical trajectory-based importance sampling methods. 

The fact that 
the variance in \eqref{eqn:varianceDR} is a sum of two terms is 
because of the assumption that samples from $\mu_0$ and $d_{\pi_0}$ are independent.
In practice they have dependency but it is possible to couple the samples from $\mu_0$ and $d_{\pi_0}$ in a certain way to even decrease the variance. 
We leave this to future work. 

\paragraph{Proposed Algorithm for Off-Policy Evaluation}
Suppose we have already get $\vhat$, an estimation of $V^\pi$ and $\what$, an estimation of $\rnd$, we can directly use equation \eqref{eqn:doubly_robust} to estimate $\Rpi$.
A detail procedure is described in Algorithm \ref{alg:doubly_robust_main}.

\begin{algorithm}[t] 
\caption{Infinite Horizon Doubly Robust Estimator}  
\label{alg:doubly_robust_main}
\begin{algorithmic} 
\STATE {\bf Input}: Transition data $\D_{\pi_0} = \{s_t^{(i)}, a_t^{(i)}, r_t^{(i)}\}_{1\leq i\leq n, 0\leq t\leq T}$ from policy $\pi_0$; a target policy $\pi$, let $\D_0 = \{s_0^{(j)}\}_{1\leq j\leq n_0}$ 
be samples from initial distribution $\mu_0$;
a good trained value function $\vhat$; a good trained density ratio $\what$. 
\STATE \textbf{Estimation:} Use $\Rdr$ in \eqref{eqn:doubly_robust} to estimate $\Rpi$ using sample from $\D$ and $\D_0$.
\end{algorithmic} 
\end{algorithm}

\section{Double Robustness and Lagrangian Duality} 
\label{sec:duality}
We reveal a surprising connection between our double robustness and Lagrangian duality. 
We show that 
our doubly robust estimator is equivalent to 
the Lagrangian function of primal dual formulation of policy evaluation. This connection is of its own interest, 
and may provide a foundation for deriving more new algorithms in future works. 

We start with the following classical optimization formulation of policy evaluation \citep{puterman2014markov}: 
\begin{align}
\label{eqn:primal}
    R^\pi = \min_V
    \left \{  (1-\gamma) \sum_{s} \mu_0(s) V(s) 
    ~~~~~~~\text{s.t. }~~~~~~~ V(s) \geq r^\pi(s) + \gamma \Ppi V(s), ~~~~~\forall s \right\}, 
\end{align}
where we find $V$ to maximize its average value, 
subject to an inequality constraint on the Bellman equation. 
It can be shown that the solution of \eqref{eqn:primal} is achieved by the true value function $V^\pi$, hence yielding an true expected reward $\Rpi$.


Introducing a Lagrangian multiplier $\rho\geq 0$, 
we can derive the Lagrangian function $L(V, \rho)$ of  \eqref{eqn:primal}, 
\begin{align} 
\label{eqn:lagrangian} 
 L(V,\rho)  =  
 (1-\gamma) \sum_{s} \mu_0(s) V(s) 
    - \sum_{s} \rho(s)(V(s) -  r^\pi(s) + \gamma \Ppi V(s)). 
\end{align}
Comparing $L(V, \rho)$ with our estimator $ \Rdr[\vhat, \what]$ in  \eqref{eqn:doubly_robust}, 
we can see that they are in fact equivalent in expectation. 
\begin{thm}
\label{thm:primal-dual}
I) Define $w_{\rho/\pi_0}(s) = \frac{\rho(s)}{d_{\pi_0}(s)}$. 
We have
\begin{align*} 
    L(V,\rho) = \Rdri[V,w_{\rho/\pi_0}], 
    &&
    \text{and hence}
    &&
    L(V, d_\pi) = L(V^\pi, \rho) = R^\pi, ~~~\text{for any $V$, $\rho$}, 
\end{align*}
which suggests that $L(V, \rho)$ is ``doubly robust'' in that it equals 
$R^\pi$ if either $V =V^\pi$ or $\rho = d_\pi$.  

II) The primal problem \eqref{eqn:primal}  forms a strong duality with the following dual problem, 
\begin{align}
\label{eqn:dual}
    R^\pi = \max_{\rho\geq 0}\left\{ 
    \sum_s \rho(s) r^\pi(s) 
    ~~~~~~~\text{s.t. }~~~~~~~ \rho(s') = (1-\gamma)\mu_0(s') + \gamma
    \mathcal T^\pi \rho(s'), 
    ~~~~~\forall s' \right\}, 
\end{align}
where $\mathcal T^\pi$ is defined in \eqref{eqn:bellman_density}. 
\end{thm}
This shows that the dual problem is equivalent to constraint $\rho$ using the fixed point equation \eqref{eqn:bellman_density} and maximize the average reward given distribution $\rho$.
Since the unique fixed point of 
\eqref{eqn:bellman_density} is $d_\pi(s)$, 
the solution of \eqref{eqn:dual} naturally yields the true reward $\Rpi$, hence forming a zero duality gap with \eqref{eqn:primal}.  


It is natural to intuitize the double robustness of the Lagrangian function.  
From \eqref{eqn:lagrangian},  $L(V, \rho)$ can be viewed 
 as estimating the reward $R^\pi$ using value function with a correction of Bellman residual $(V - r^\pi - \gamma \Ppi V)$. 
 If $V = V^\pi$, the estimation equals the true reward and the correction equals zero. 
From the dual problem \eqref{eqn:dual},  $L(V, \rho)$ can be viewed as  estimating $R^\pi$ using density function $\rho$, corrected by the residual $(\rho - (1-\gamma)\mu_0 - \gamma \mathcal T^\pi \rho)$. 
We again get the true reward if $\rho  = d_\pi$. 

It turns out that we can use the primal-dual formula when $\gamma = 1$ to obtain the double robust estimator for the average reward case.
We clarify it in appendix \ref{sec:avg_prime}.
\paragraph{Remark} 
The fact that the density function $d_\pi$ forms a dual variable of the value function $V^\pi$ is widely known in the optimal control and reinforcement learning literature~\citep[e.g.,][]{bertsekas1995dynamic,puterman2014markov,defarias03linear}, and has been leveraged in various works for policy optimization. 
 However, it does not seem to be well exploited in the literature of off-policy policy evaluation.

\section{Related Work}

\paragraph{Off-Policy Value Evaluation}
The problem of off-policy value evaluation has been studied in contextual bandits \citep{dudik11doubly,wang17optimal} and more general finite horizon RL settings \citep{fonteneau13batch, li2015toward,jiang16doubly, thomas16data, liu2018representation, farajtabar18more,xie2019optimal}.
However, most of the existing works are based on importance sampling (IS) to correct the mismatch between the distribution of the whole trajectories induced by the behavior and target policies, which faces the ``curse of horizon'' \citep{liu2018breaking} when extended to long-horizon (or infinite-horizon) problems.

Several other works \citep{guo17using,hallak17consistent, liu2018breaking,gelada2019off,nachum2019dualdice} have been proposed to address the high variance issue in the long-horizon problems.
\citet{liu2018breaking} apply importance sampling on the average visitation distribution of
state-action pairs, 
instead of the distribution of the whole trajectories, 
which provides a unified approach to break ``the curse of horizon''.
However, they require to learn a density ratio function over the whole state-action pairs, which may induce large bias.
Our work incorporates the density ratio and value function estimation, which significantly reduces the induced bias of two estimators, resulting a doubly robust estimator.

Our work is also closely related to DR techniques used in finite horizon problems \citep{murphy01marginal,dudik11doubly,jiang16doubly,thomas16data,farajtabar18more}, which incorporate an approximate value function as control variates to IS estimators.
Different from existing DR approaches, our work is related to the well known duality between the density and the value function, 
which reveals the relationship between density (ratio) learning \citep{liu2018breaking} and value function learning.  
Based on this interesting observation, 
we further obtain the doubly robust estimator for estimating average reward in infinite-horizon problems.



\paragraph{Primal-Dual Value Learning} 
Primal-dual optimization techniques have been widely used for off-policy value function learning and policy optimization \citep{liu2015finite,chen2016stochastic,dai17learning,dai2017sbeed,feng2019kernel}.
Nevertheless, the duality between density and value function has not been well explored in the literature of off policy value estimation.
Our work proposes a new doubly robustness technique for off-policy value estimation, which can be naturally viewed as the Lagrangian function of the primal-dual formulation of policy evaluation, providing an alternative unified view for off policy value evaluation.






\section{Experiment}

\newcommand{\taxlen}{.285\linewidth}
\newcommand{\gapline}{-.03\linewidth}
\begin{figure}[t]
    \centering
    \begin{tabular}{ccc}
        \multicolumn{3}{c}{
        \includegraphics[width=.95\textwidth]{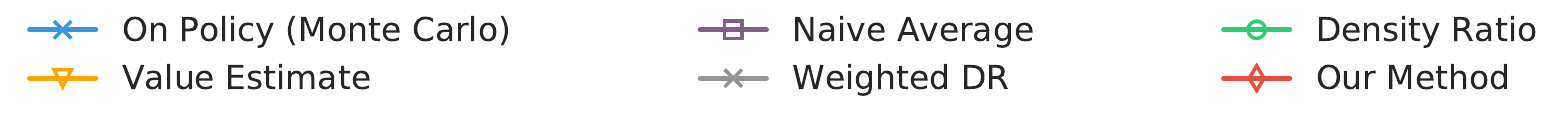}}\\
        \includegraphics[height=\taxlen]{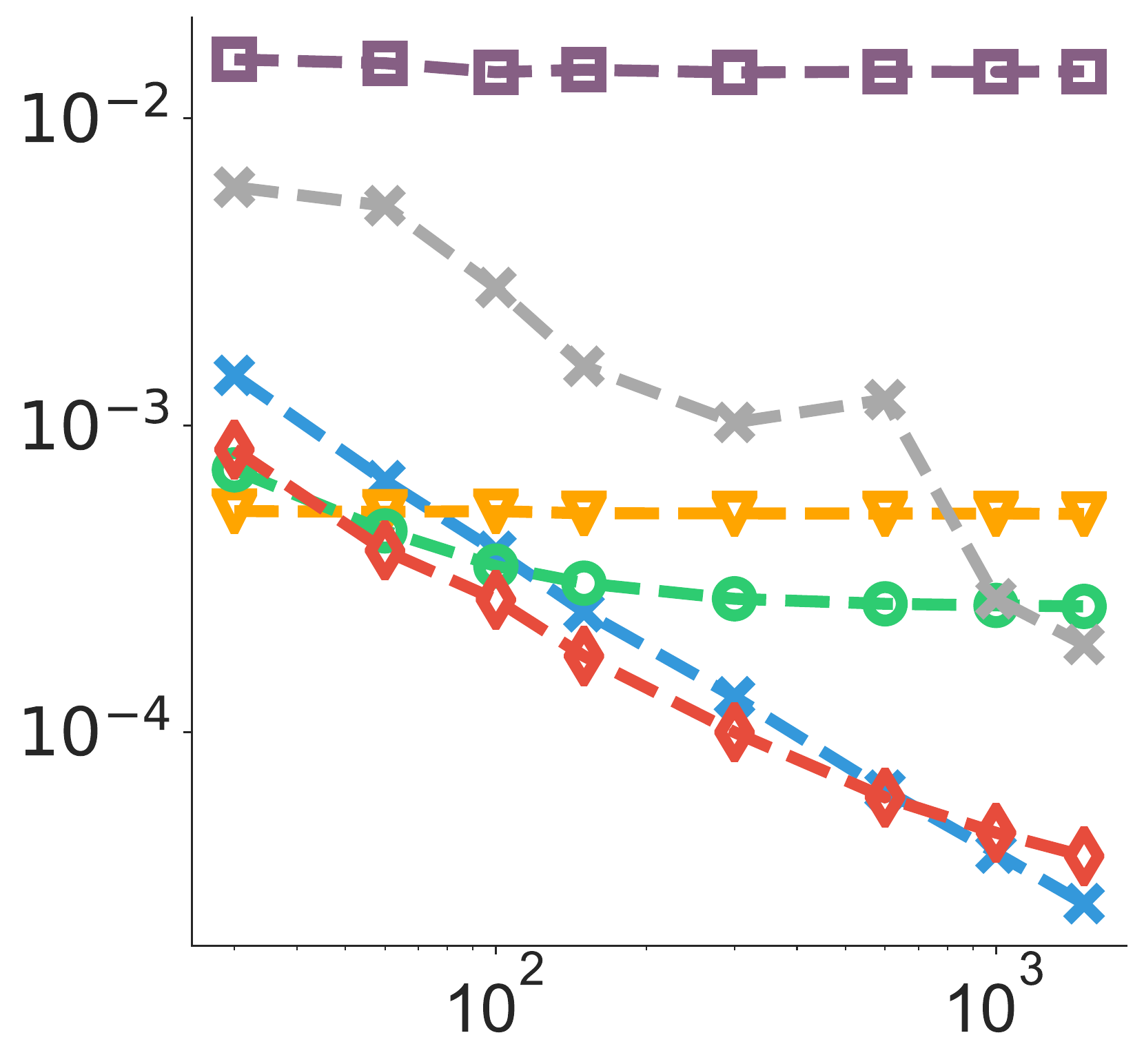}& 
        \hspace{\gapline}
        \includegraphics[height=\taxlen]{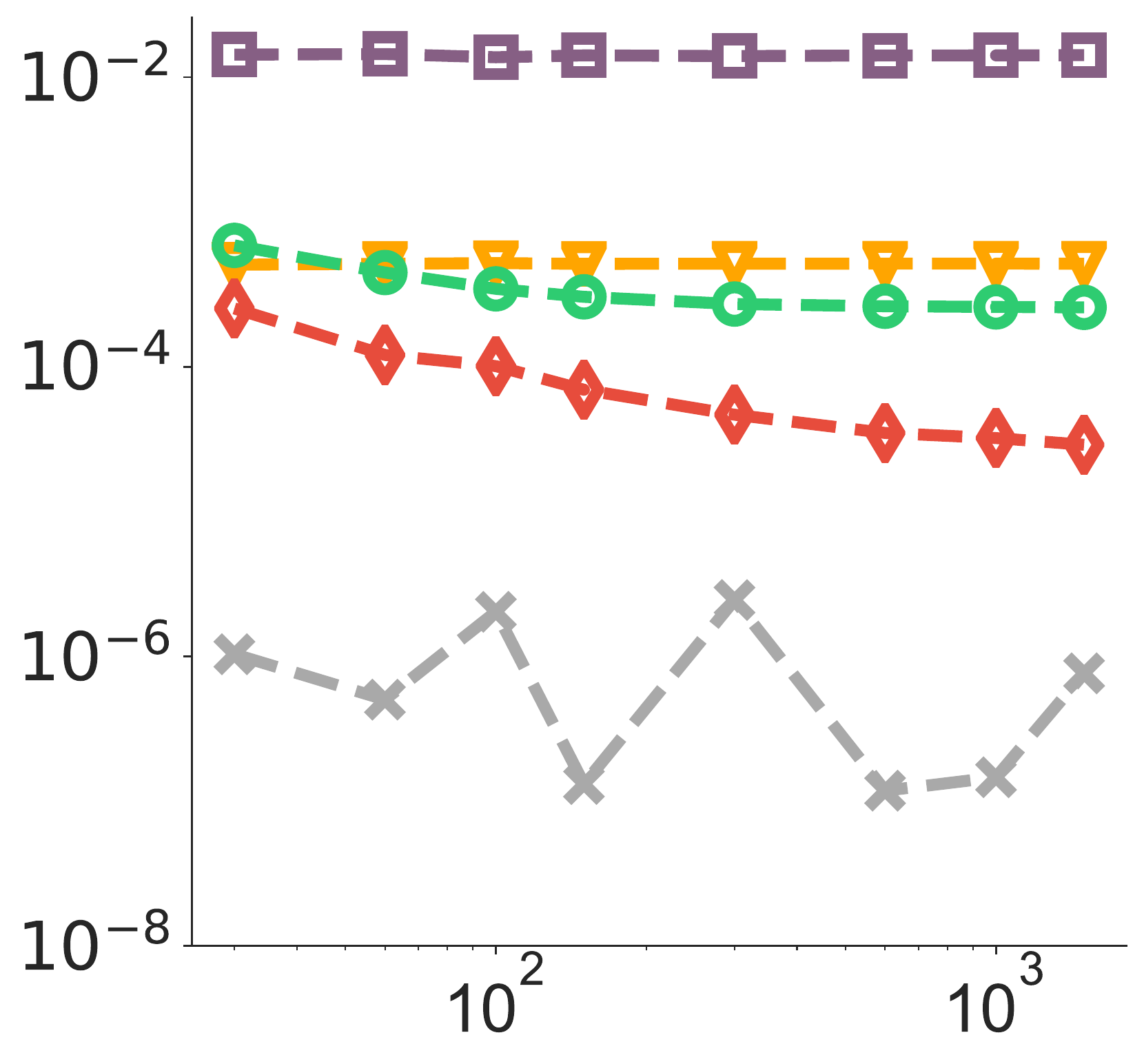}&
        \hspace{\gapline}
        \includegraphics[height=\taxlen]{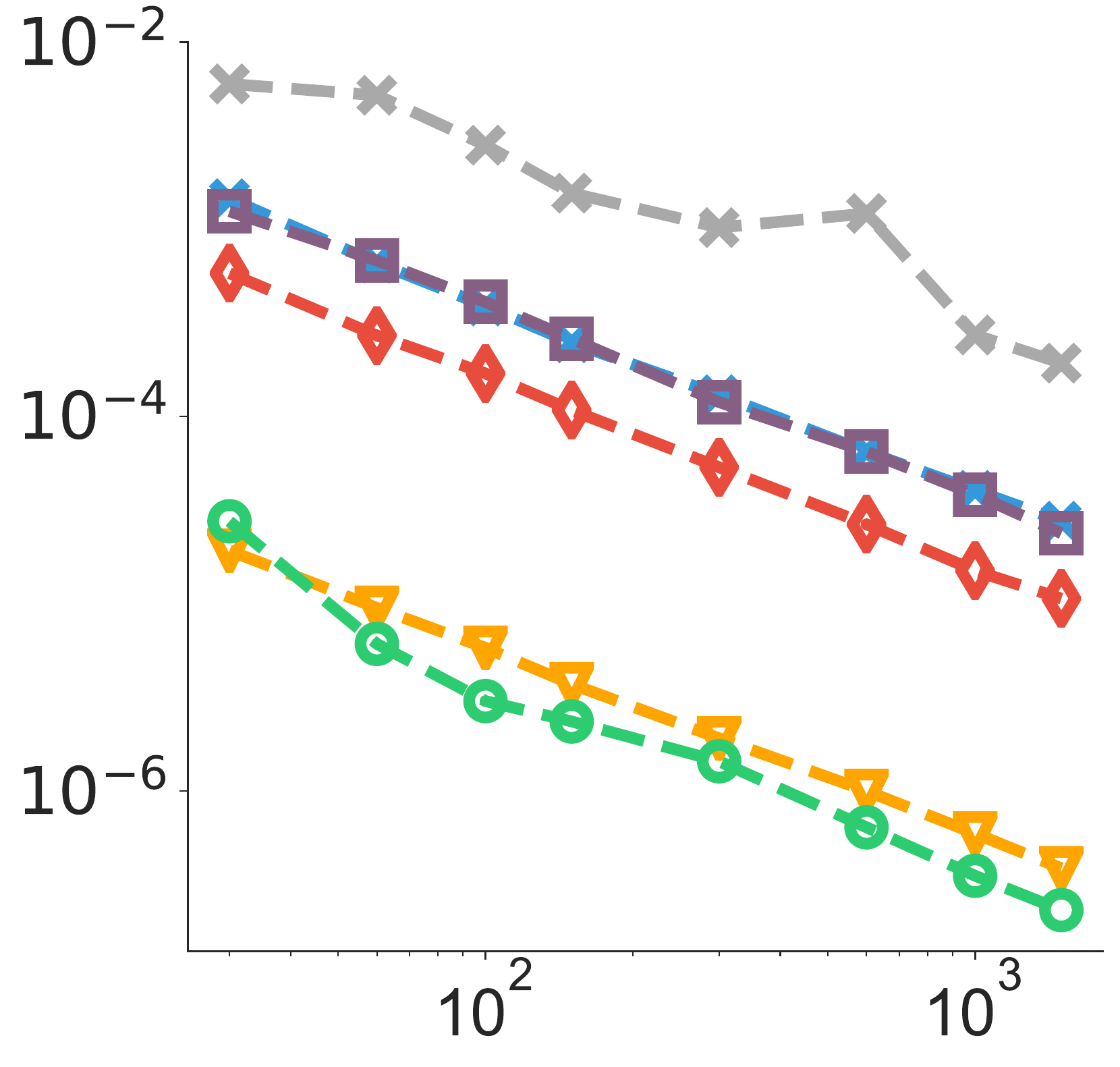}\\
        
        (a) \small{MSE} & (b) \small{Bias Square} & (c) \small{Variance}\\
         \includegraphics[height=\taxlen]{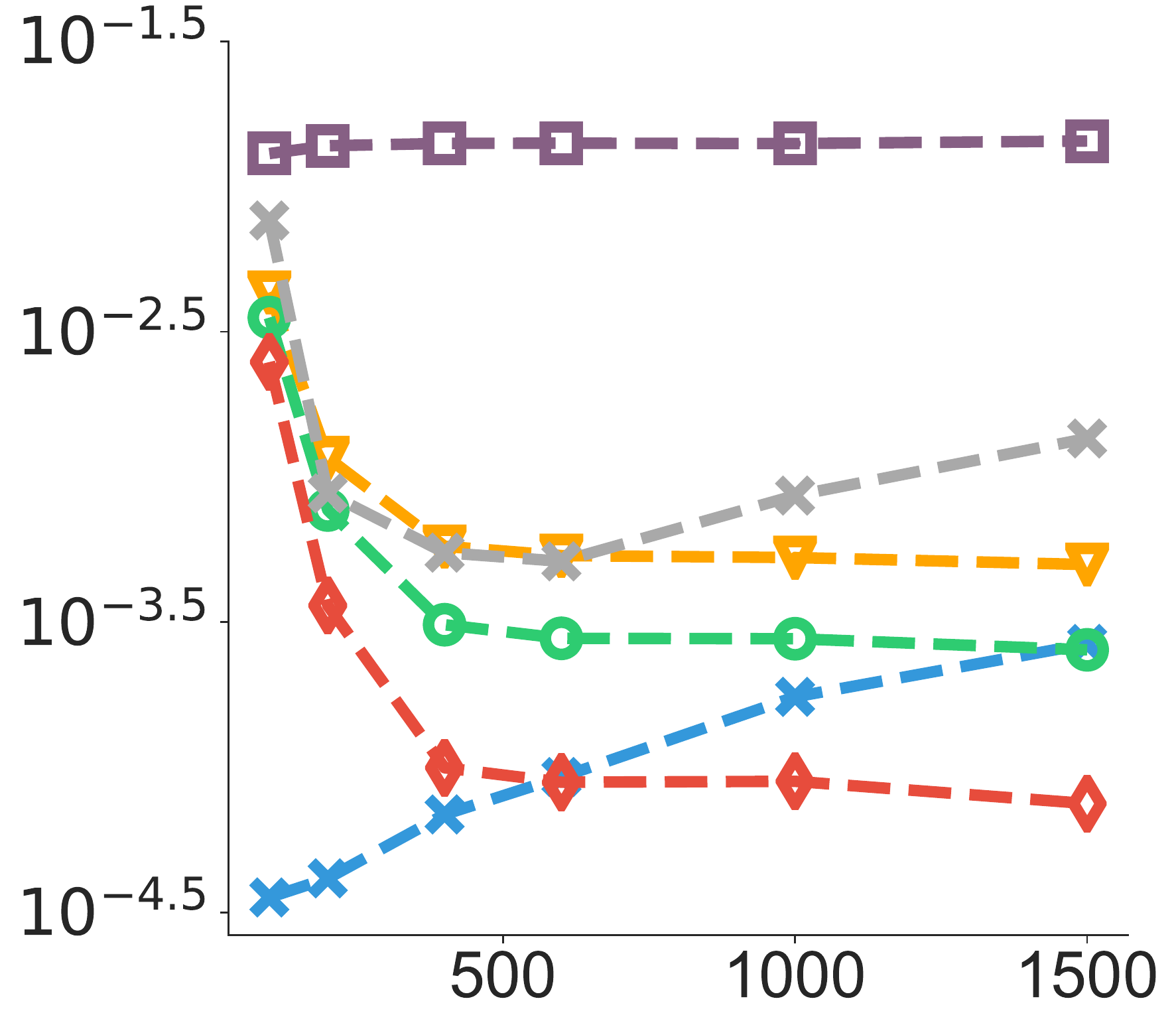}&
         \hspace{\gapline}
         \includegraphics[height=\taxlen]{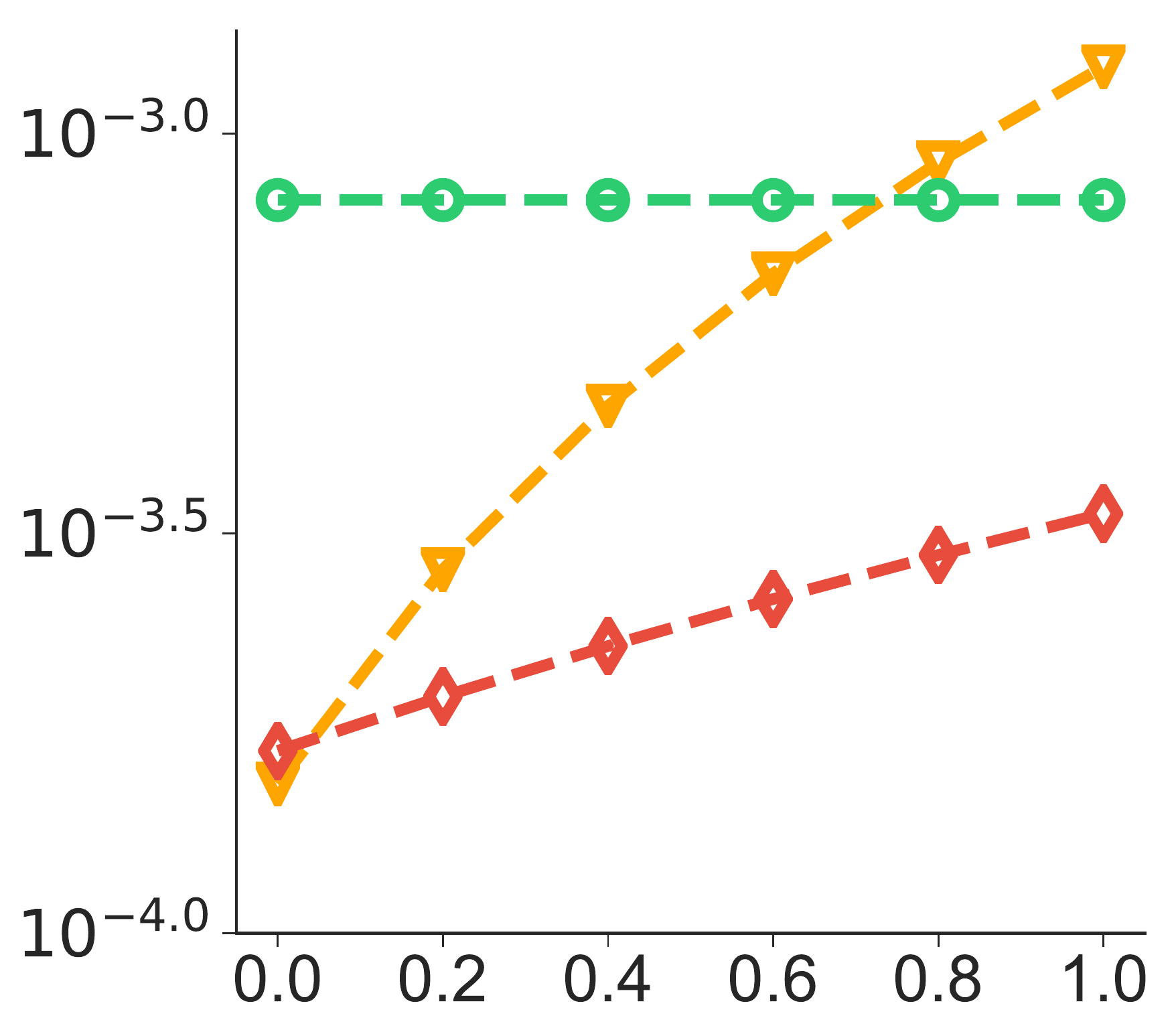}&
         \hspace{\gapline}
         \includegraphics[height=\taxlen]{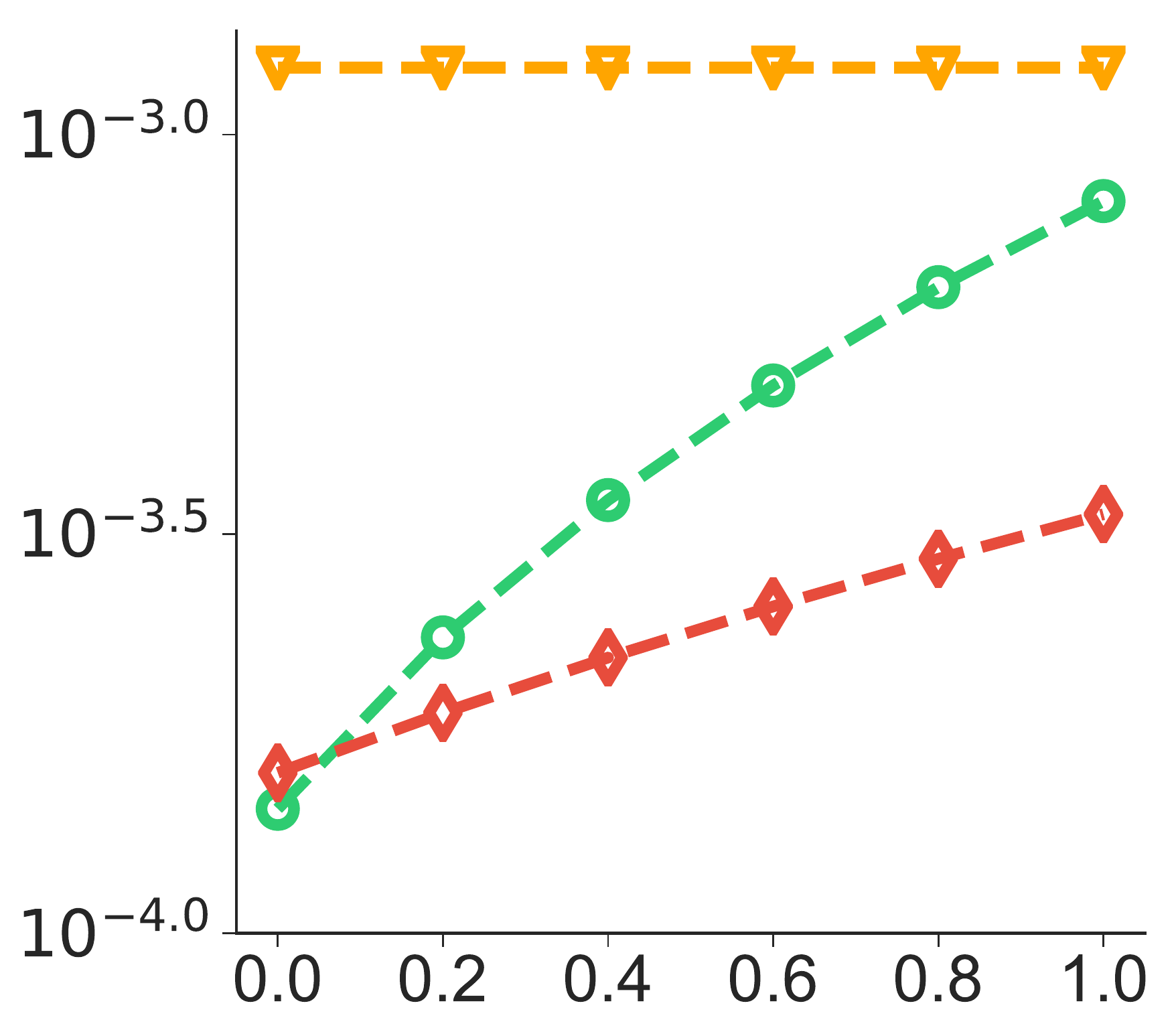}\\
       (d) \small{MSE with $H$ changes} & (e) \small{Bias square with $\alpha$ changes} & (f) \small{Bias square with $\beta$ changes} \\
    \end{tabular}
    \caption{Off Policy Evaluation Results on Taxi. Default parameter, discounted factor $\gamma = 0.99$, mixed ratio $\alpha = \beta = 1$, horizon length $H = 600$. For (a)-(c) the x-axis is the number of trajectories and y-axis corresponds to MSE, Bias Square and Variance, respectively. For (d) we fix the total number of samples (number of trajectories times horizon length) and change the horizon length as x-axis and observe the MSE. (e) and (f) show the change the mixed ratio of $\alpha$, $\beta$ with the change of bias. We repeat each experiment for 1000 runs.}
    \label{fig:taxi_evaluation}
\end{figure}

In this section, we conduct simulation experiments on different environmental settings to compare our new doubly robust estimator with existing methods.
We mainly compare with infinite horizon based estimator including state importance sampling estimator (\cite{liu2018breaking}) and value function estimator.
We do not report results on the vanilla trajectory-based importance sampling estimators because of their significant higher variance,
but we do compare with the doubly robust version induced by \citet{thomas16data} (self-normalized variant of \citet{jiang16doubly}).
In all experiments we compare with Monte Carlo and naive average as \citet{liu2018breaking} suggested.
The ground truth for each environment is calculated by averaging Monte Carlo estimation with a very large sample size. 

\paragraph{Taxi Environment}

We follow \citet{liu2018breaking}'s tabular environment \textit{Taxi},
which has $2000$ states and $6$ actions in total.
For more experimental details, please check appendix \ref{sec:taxi}.

We pre-train two different $\vhat$ and $\tilde{V}$ trained with a small and fairly large  size of samples, respectively, where $\tilde{V}$ is very close to true value function $V^\pi$ but $\vhat$ is relatively further from it.
Similarly we pre-train $\widehat{\rho}$ and $\tilde{\rho} \approx d_\pi$.
For estimation we use a mixed ratio $\alpha, \beta$ to control the bias of the input $V, \rho$, where $V = \alpha \vhat + (1-\alpha)\tilde{V}$ and $\rho = \beta \widehat{\rho} + (1-\beta)\tilde{\rho}$.

Figure \ref{fig:taxi_evaluation}(a)-(c) show results of comparison for different methods as we changing the number of trajectories.
We can see that the MSE performance of value function($\Rval$) and state visitation importance sampling($\Rsis$) estimators are mainly impeded by their large biases,
while our method has much less bias thus it can keep decreasing as sample size increase and achieves same performance as on policy estimator.
Figure \ref{fig:taxi_evaluation}(d) shows results if we change the horizon length. 
Notice that here we keep the number of samples to be the same, so if we increase our horizon length we will decrease the number of trajectories in the same time.
We can see that our method alongside with all infinite horizon methods will get better result as horizon length increase.
Figure \ref{fig:taxi_evaluation}(e)-(f) indicate the ``double robustness'' of our method, 
where our method benefits from either a better $V$ or a better $\rho$.


    
\paragraph{Puck-Mountain}
Puck-Mountain is an environment similar to Mountain-Car, except that the goal of Puck-Mountain is to push the puck as high as possible in a local valley, which has a continuous state space of $\mathbb{R}^{2}$ and a discrete action space similar to Mountain-Car. 
We use the softmax functions of an optimal Q-function as both target policy and behavior policy, where the temperature of the behavior policy is higher (encouraging exploration).
For more details of constructing policies and training algorithms for density ratio and value functions, please check appendix \ref{sec:exp_continuous}.



Figure \ref{fig:puck_eval}(a)-(c) show results of comparison for different methods as we changing the number of trajectories.
Similar to taxi, we find our method has much lower bias than density ratio and value function estimation, which yields a better MSE.
In Figure \ref{fig:puck_eval}(d) the performance for all infinite horizon estimator will not degenerate as horizon increases,
while finite horizon method such as finite weighted horizon doubly robust will suffer from larger variance as horizon increases.

\newcommand{\penlen}{.24\linewidth}
\newcommand{\cgapline}{-.035\linewidth}
\begin{figure}[t]
    \centering
     \begin{tabular}{cccc}
        \multicolumn{4}{c}{
        \includegraphics[width=.98\textwidth]{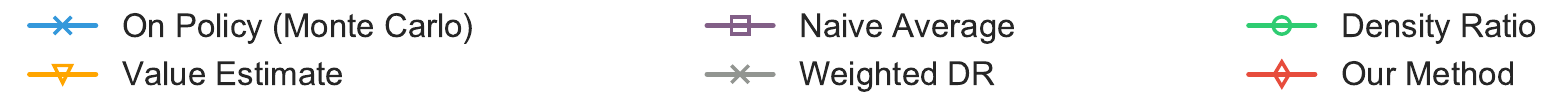}}\\
        \hspace{\cgapline}
        \includegraphics[height=\penlen]{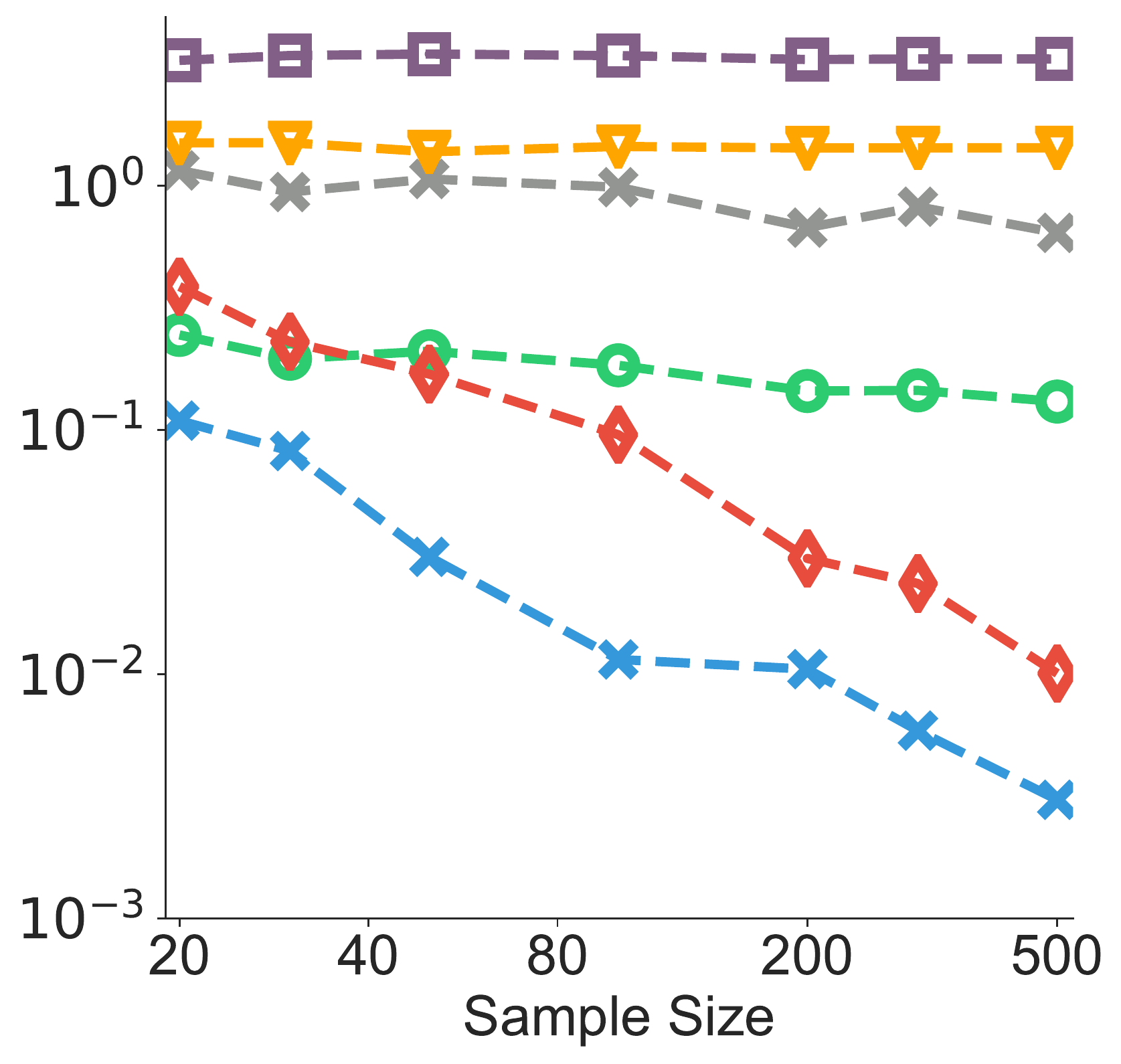}&
        \hspace{\cgapline}
        \includegraphics[height=\penlen]{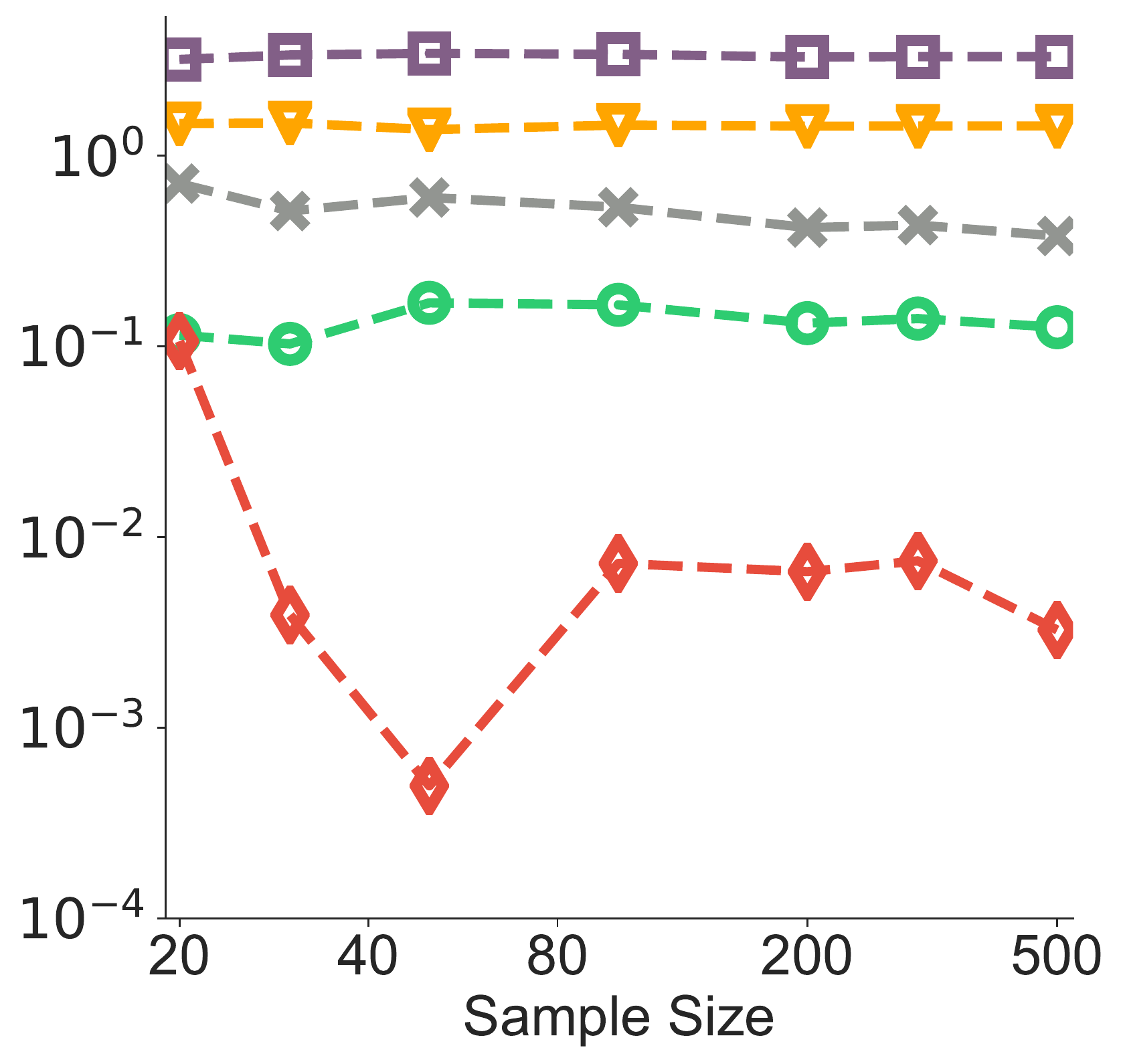}&
        \hspace{\cgapline}
        \includegraphics[height=\penlen]{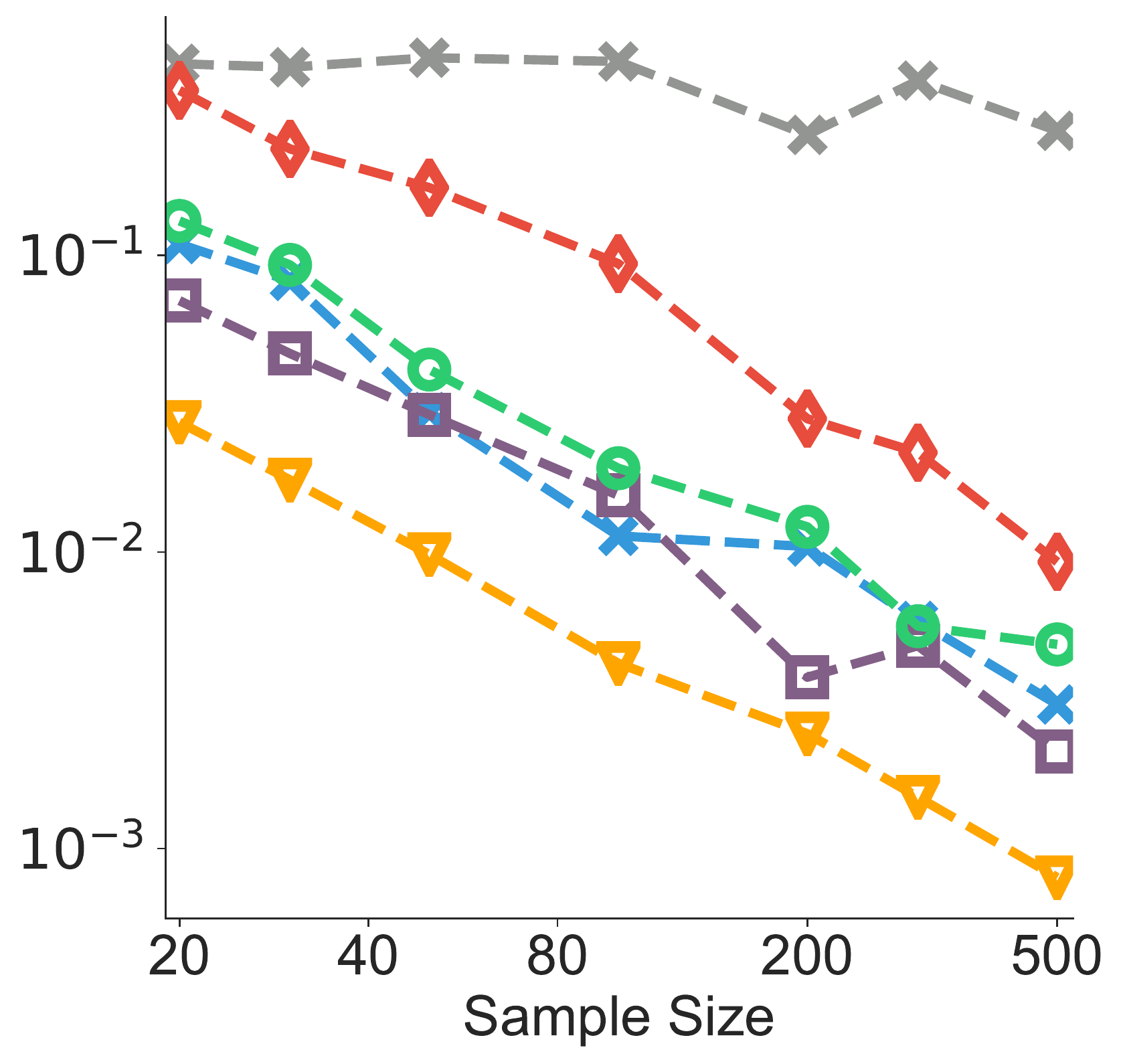}&
        \hspace{\cgapline}
         \includegraphics[height=\penlen]{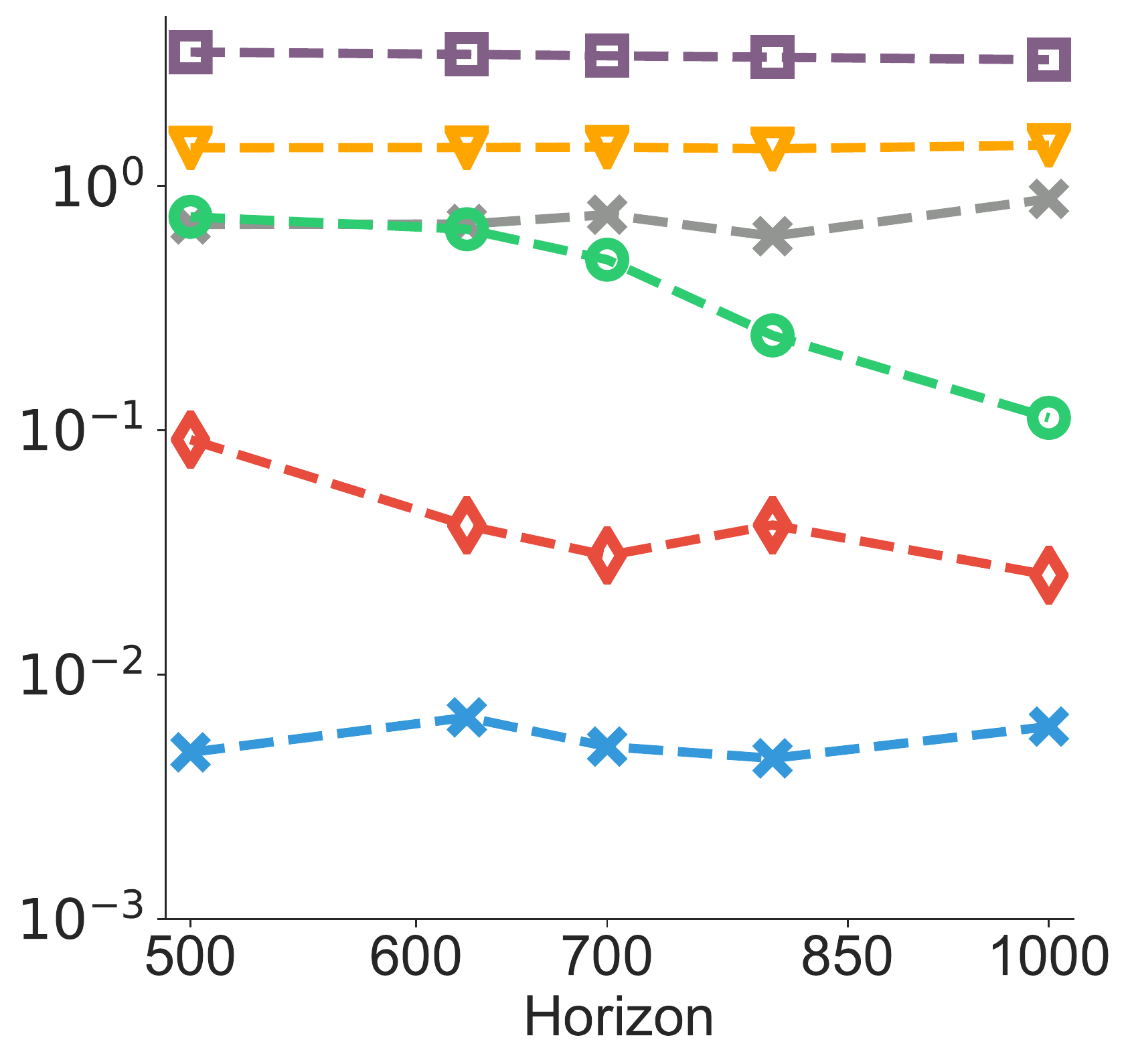}\\
        (a) \small{MSE} & (b) \small{Bias Square} & (c) \small{Variance} & (d) \small{MSE with $H$ changes} \\
    \end{tabular}
    \caption{Off Policy Evaluation Results on Puck-Mountain. We set discounted factor $\gamma = 0.995$ as default. For (a)-(c) we set the horizon $H = 1000$ and the x-axis is the number of trajectories for used for evaluation. For (d) we fix the total number of samples and change the horizon length.}
    \label{fig:puck_eval}
\end{figure}

\begin{figure}[t]
    \centering
    \begin{tabular}{cccc}
        \multicolumn{4}{c}{
        \includegraphics[width=1.00\textwidth]{figures/inverted_all/legend_bench.pdf}}\\
        \hspace{\gapline}
        \includegraphics[height=\penlen]{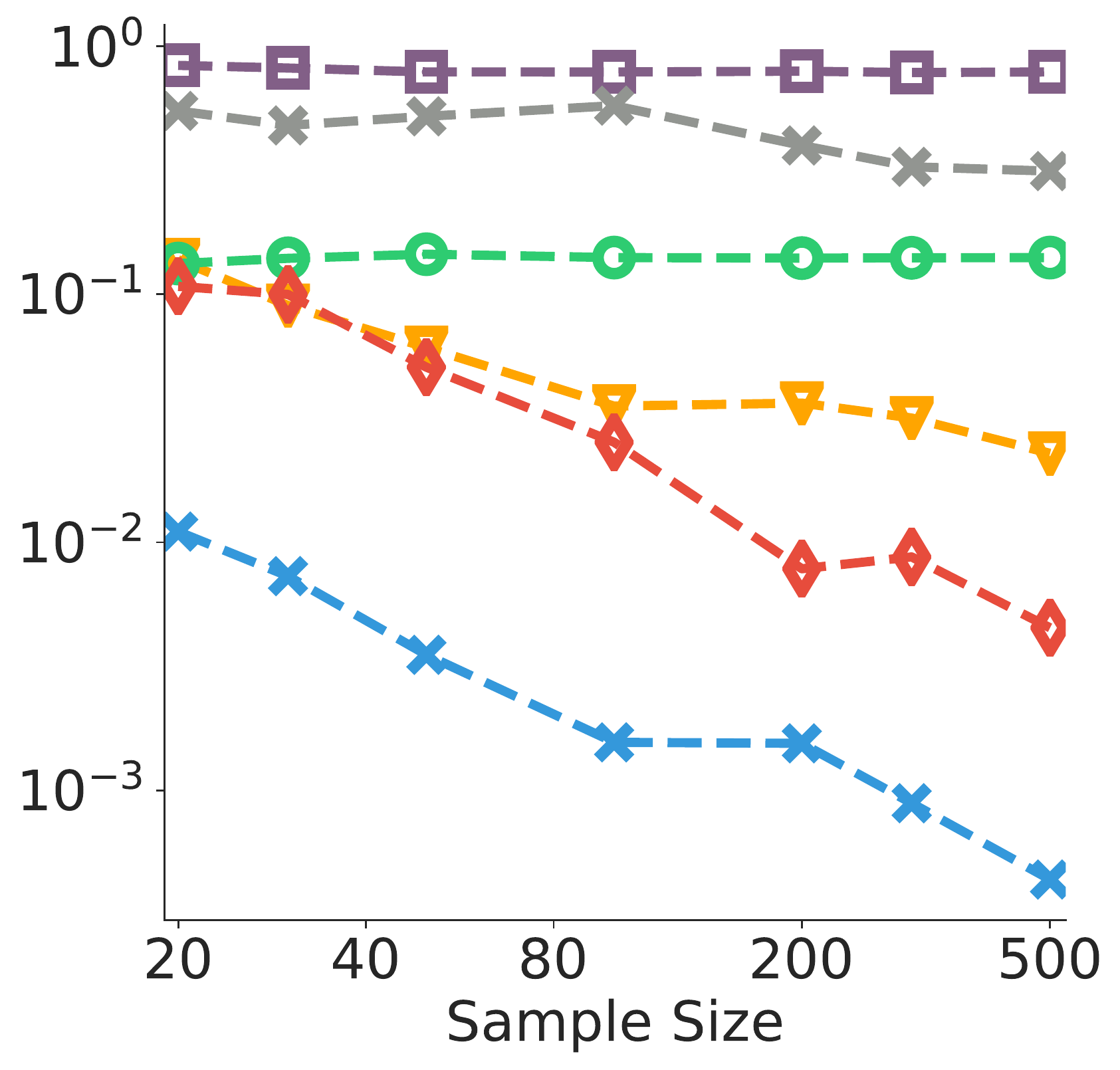}& 
        \hspace{\cgapline}
        \includegraphics[height=\penlen]{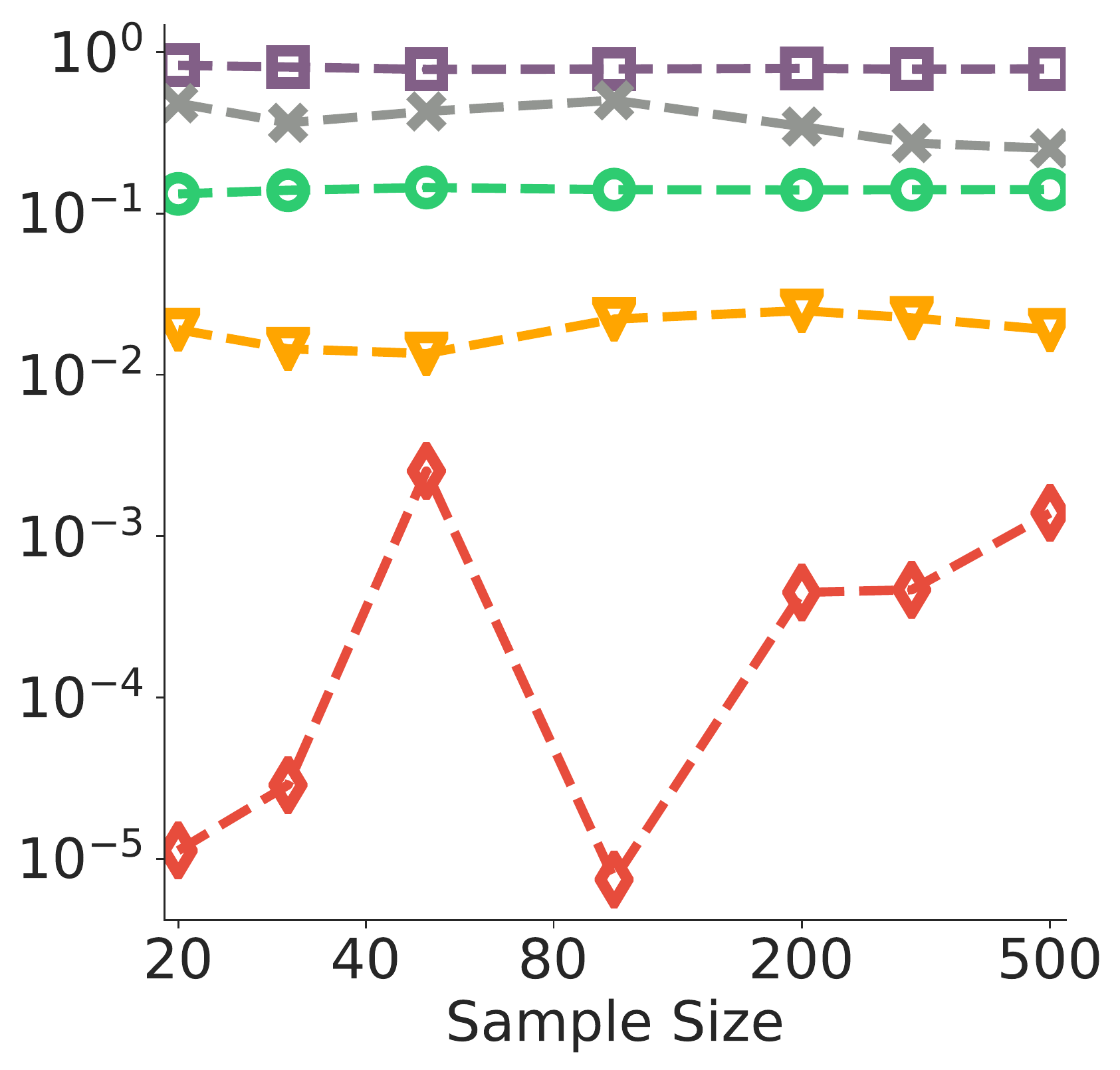}&
        \hspace{\cgapline}
        \includegraphics[height=\penlen]{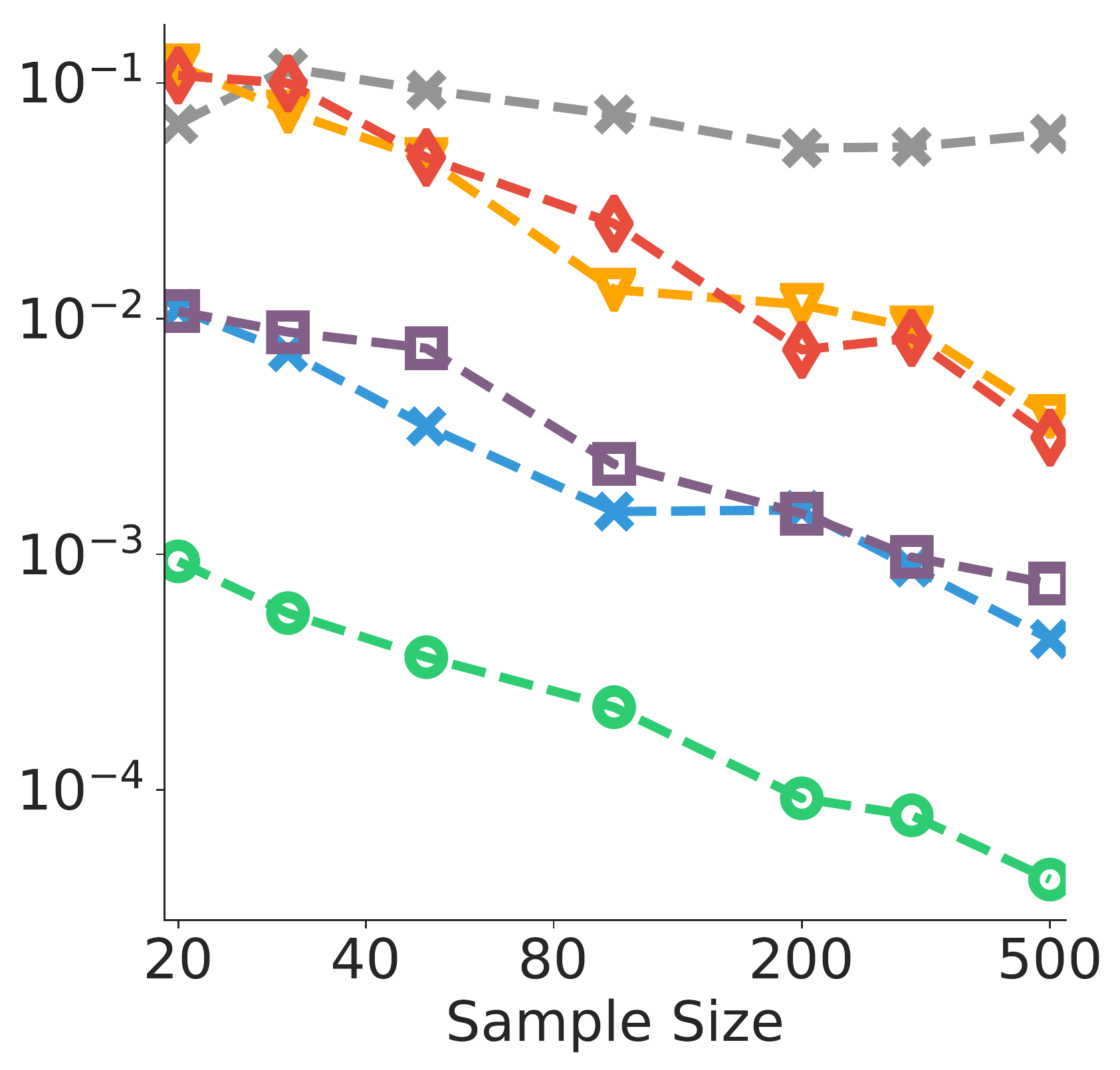}&
        \hspace{\cgapline}
         \includegraphics[height=\penlen]{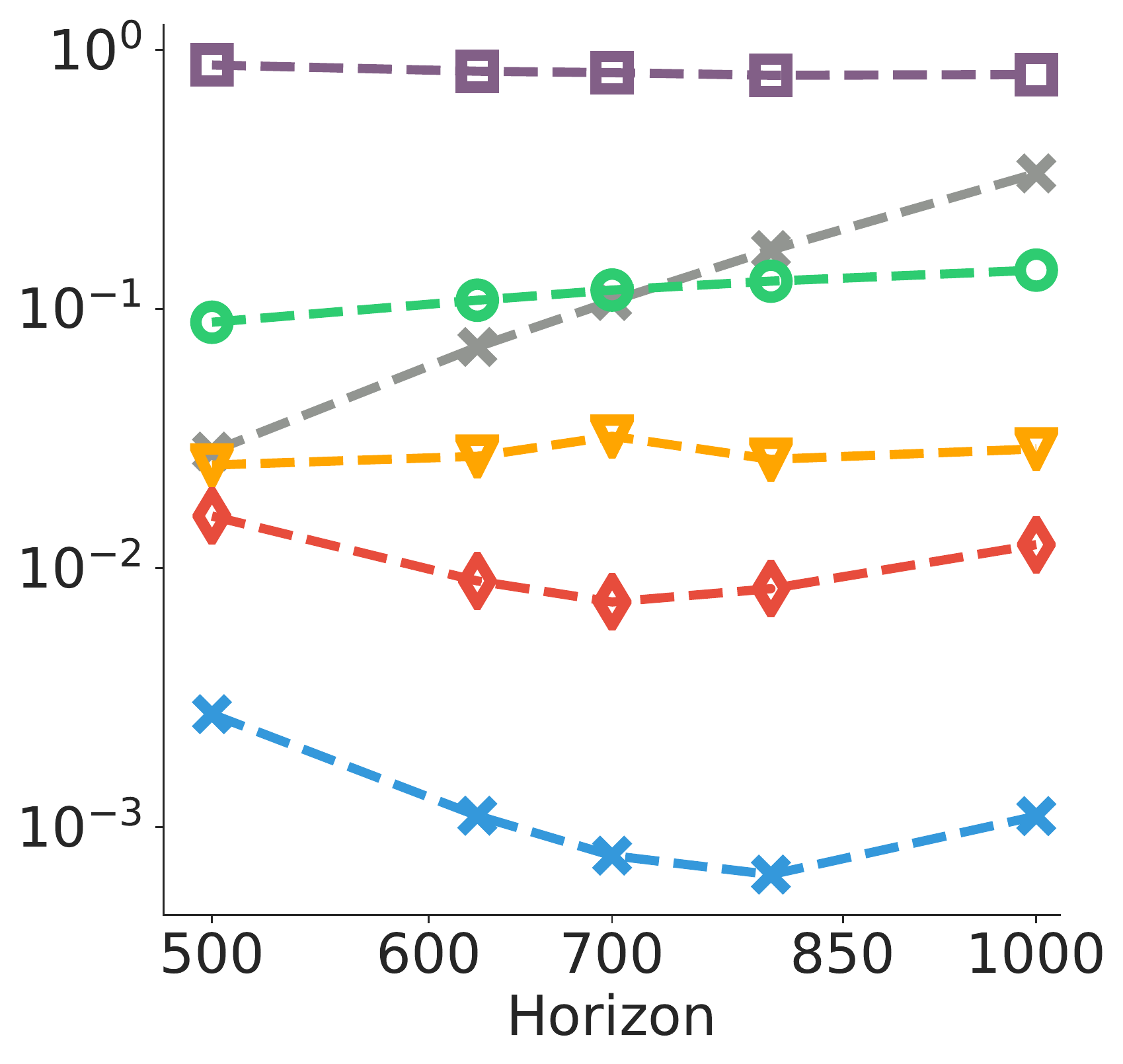}\\
        (a) \small{MSE} & (b) \small{Bias Square} & (c) \small{Variance} & (d) \small{MSE with $H$ changes} 
        \vspace{-.em}\\
    \end{tabular}
    \vspace{-0.0em}
    \caption{Off Policy Evaluation Results on InvertedPendulum-v2. We set discounted factor $\gamma = 0.995$ as default. For (a)-(c) we set the horizon $H = 1000$ and the x-axis is the number of trajectories for used for evaluation. For (d) we fix the total number of samples and change the horizon length.}
    \label{fig:InvertedPendulum_evaluation}
\end{figure}

\paragraph{InvertedPendulum}
InvertedPendulum is a pendulum that has its center of mass above its pivot point. 
We use the implementation of InvertedPendulum from OpenAI gym \citep{brockman2016openai}, which is a continuous control task with state space in $\R^4$ and we discrete the action space to be $\{-1, -0.3, -0.2, 0, 0.2, 0.3, 1\}$.
More experiment details can be found in appendix \ref{sec:exp_continuous}.

In Figure \ref{fig:InvertedPendulum_evaluation}(a)-(c) our method  again significantly reduces the bias, which yields a better MSE comparing with value and density estimation. 
Figure \ref{fig:InvertedPendulum_evaluation}(d) also shows that our method consistently outperforms all other methods as the horizon increases with a fixed total timesteps.



\section{Conclusion}

In this paper, we develop a new doubly robust estimator based on the infinite horizon density ratio and off policy value estimation.
Our new proposed doubly robust estimator can be accurate as long as one of the estimators are accurate, which yields a significant advantage comparing to previous estimators. 
Future directions include deriving more novel optimization algorithms to learn value function and density(ratio) function by using the primal dual framework.

\clearpage
\bibliography{reference}

\begin{thebibliography}{32}
\providecommand{\natexlab}[1]{#1}
\providecommand{\url}[1]{\texttt{#1}}
\expandafter\ifx\csname urlstyle\endcsname\relax
  \providecommand{\doi}[1]{doi: #1}\else
  \providecommand{\doi}{doi: \begingroup \urlstyle{rm}\Url}\fi

\bibitem[Asmussen \& Glynn(2007)Asmussen and Glynn]{asmussen2007stochastic}
S{\o}ren Asmussen and Peter~W Glynn.
\newblock \emph{Stochastic simulation: algorithms and analysis}, volume~57.
\newblock Springer Science \& Business Media, 2007.

\bibitem[Bertsekas(2000)]{bertsekas1995dynamic}
Dimitri~P. Bertsekas.
\newblock \emph{Dynamic Programming and Optimal Control}.
\newblock Athena Scientific, 2nd edition, 2000.
\newblock ISBN 1886529094.

\bibitem[Bottou et~al.(2013)Bottou, Peters, Qui{\~n}onero-Candela, Charles,
  Chickering, Portugaly, Ray, Simard, and Snelson]{bottou13counterfactual}
L\'{e}on Bottou, Jonas Peters, Joaquin Qui{\~n}onero-Candela, Denis~Xavier
  Charles, D.~Max Chickering, Elon Portugaly, Dipankar Ray, Patrice Simard, and
  Ed~Snelson.
\newblock Counterfactual reasoning and learning systems: The example of
  computational advertising.
\newblock \emph{Journal of Machine Learning Research}, 14:\penalty0 3207--3260,
  2013.

\bibitem[Brockman et~al.(2016)Brockman, Cheung, Pettersson, Schneider,
  Schulman, Tang, and Zaremba]{brockman2016openai}
Greg Brockman, Vicki Cheung, Ludwig Pettersson, Jonas Schneider, John Schulman,
  Jie Tang, and Wojciech Zaremba.
\newblock Openai gym.
\newblock \emph{arXiv preprint arXiv:1606.01540}, 2016.

\bibitem[Chen \& Wang(2016)Chen and Wang]{chen2016stochastic}
Yichen Chen and Mengdi Wang.
\newblock Stochastic primal-dual methods and sample complexity of reinforcement
  learning.
\newblock \emph{arXiv preprint arXiv:1612.02516}, 2016.

\bibitem[Dai et~al.(2017{\natexlab{a}})Dai, He, Pan, Boots, and
  Song]{dai17learning}
Bo~Dai, Niao He, Yunpeng Pan, Byron Boots, and Le~Song.
\newblock Learning from conditional distributions via dual kernel embeddings.
\newblock In \emph{Proceedings of the 20th International Conference on
  Artificial Intelligence and Statistics (AISTATS)}, pp.\  1458--1467,
  2017{\natexlab{a}}.
\newblock CoRR abs/1607.04579.

\bibitem[Dai et~al.(2017{\natexlab{b}})Dai, Shaw, Li, Xiao, He, Liu, Chen, and
  Song]{dai2017sbeed}
Bo~Dai, Albert Shaw, Lihong Li, Lin Xiao, Niao He, Zhen Liu, Jianshu Chen, and
  Le~Song.
\newblock Sbeed: Convergent reinforcement learning with nonlinear function
  approximation.
\newblock \emph{arXiv preprint arXiv:1712.10285}, 2017{\natexlab{b}}.

\bibitem[{de Farias} \& {Van Roy}(2003){de Farias} and {Van
  Roy}]{defarias03linear}
Daniela~Pucci {de Farias} and Benjamin {Van Roy}.
\newblock The linear programming approach to approximate dynamic programming.
\newblock \emph{Operations Research}, 51\penalty0 (6):\penalty0 850--865, 2003.

\bibitem[Dud\'ik et~al.(2011)Dud\'ik, Langford, and Li]{dudik11doubly}
Miroslav Dud\'ik, John Langford, and Lihong Li.
\newblock Doubly robust policy evaluation and learning.
\newblock In \emph{Proceedings of the 28th International Conference on Machine
  Learning (ICML)}, pp.\  1097--1104, 2011.

\bibitem[Farajtabar et~al.(2018)Farajtabar, Chow, and
  Ghavamzadeh]{farajtabar18more}
Mehrdad Farajtabar, Yinlam Chow, and Mohammad Ghavamzadeh.
\newblock More robust doubly robust off-policy evaluation.
\newblock In \emph{Proceedings of the 35th International Conference on Machine
  Learning (ICML)}, pp.\  1446--1455, 2018.

\bibitem[Feng et~al.(2019)Feng, Li, and Liu]{feng2019kernel}
Yihao Feng, Lihong Li, and Qiang Liu.
\newblock A kernel loss for solving the bellman equation.
\newblock \emph{Neural Information Processing Systems (NeurIPS)}, 2019.

\bibitem[Fonteneau et~al.(2013)Fonteneau, Murphy, Wehenkel, and
  Ernst]{fonteneau13batch}
Raphael Fonteneau, Susan~A. Murphy, Louis Wehenkel, and Damien Ernst.
\newblock Batch mode reinforcement learning based on the synthesis of
  artificial trajectories.
\newblock \emph{Annals of Operations Research}, 208\penalty0 (1):\penalty0
  383--416, 2013.

\bibitem[Gelada \& Bellemare(2019)Gelada and Bellemare]{gelada2019off}
Carles Gelada and Marc~G Bellemare.
\newblock Off-policy deep reinforcement learning by bootstrapping the covariate
  shift.
\newblock In \emph{Proceedings of the AAAI Conference on Artificial
  Intelligence}, volume~33, pp.\  3647--3655, 2019.

\bibitem[Guo et~al.(2017)Guo, Thomas, and Brunskill]{guo17using}
Zhaohan Guo, Philip~S. Thomas, and Emma Brunskill.
\newblock Using options and covariance testing for long horizon off-policy
  policy evaluation.
\newblock In \emph{Advances in Neural Information Processing Systems 30
  (NIPS)}, pp.\  2489--2498, 2017.

\bibitem[Hallak \& Mannor(2017)Hallak and Mannor]{hallak17consistent}
Assaf Hallak and Shie Mannor.
\newblock Consistent on-line off-policy evaluation.
\newblock In \emph{Proceedings of the 34th International Conference on Machine
  Learning (ICML)}, pp.\  1372--1383, 2017.

\bibitem[Jiang \& Li(2016)Jiang and Li]{jiang16doubly}
Nan Jiang and Lihong Li.
\newblock Doubly robust off-policy evaluation for reinforcement learning.
\newblock In \emph{Proceedings of the 23rd International Conference on Machine
  Learning (ICML)}, pp.\  652--661, 2016.

\bibitem[Li et~al.(2011)Li, Chu, Langford, and Wang]{li11unbiased}
Lihong Li, Wei Chu, John Langford, and Xuanhui Wang.
\newblock Unbiased offline evaluation of contextual-bandit-based news article
  recommendation algorithms.
\newblock In \emph{Proceedings of the 4th International Conference on Web
  Search and Data Mining (WSDM)}, pp.\  297--306, 2011.

\bibitem[Li et~al.(2015)Li, Munos, and Szepesv{\'a}ri]{li2015toward}
Lihong Li, R{\'e}mi Munos, and Csaba Szepesv{\'a}ri.
\newblock Toward minimax off-policy value estimation.
\newblock In \emph{Proceedings of the 18th International Conference on
  Artificial Intelligence and Statistics (AISTATS)}, pp.\  608--616, 2015.

\bibitem[Liu et~al.(2015)Liu, Liu, Ghavamzadeh, Mahadevan, and
  Petrik]{liu2015finite}
Bo~Liu, Ji~Liu, Mohammad Ghavamzadeh, Sridhar Mahadevan, and Marek Petrik.
\newblock Finite-sample analysis of proximal gradient td algorithms.
\newblock In \emph{UAI}, pp.\  504--513. Citeseer, 2015.

\bibitem[Liu(2001)]{liu01monte}
Jun~S. Liu.
\newblock \emph{{Monte Carlo} Strategies in Scientific Computing}.
\newblock Springer Series in Statistics. Springer-Verlag, 2001.
\newblock ISBN 0387763694.

\bibitem[Liu et~al.(2018{\natexlab{a}})Liu, Li, Tang, and
  Zhou]{liu2018breaking}
Qiang Liu, Lihong Li, Ziyang Tang, and Dengyong Zhou.
\newblock Breaking the curse of horizon: Infinite-horizon off-policy
  estimation.
\newblock In \emph{Advances in Neural Information Processing Systems}, pp.\
  5361--5371, 2018{\natexlab{a}}.

\bibitem[Liu et~al.(2018{\natexlab{b}})Liu, Gottesman, Raghu, Komorowski,
  Faisal, Doshi-Velez, and Brunskill]{liu2018representation}
Yao Liu, Omer Gottesman, Aniruddh Raghu, Matthieu Komorowski, Aldo~A Faisal,
  Finale Doshi-Velez, and Emma Brunskill.
\newblock Representation balancing mdps for off-policy policy evaluation.
\newblock In \emph{Advances in Neural Information Processing Systems}, pp.\
  2644--2653, 2018{\natexlab{b}}.

\bibitem[Liu et~al.(2019)Liu, Swaminathan, Agarwal, and Brunskill]{liu2019off}
Yao Liu, Adith Swaminathan, Alekh Agarwal, and Emma Brunskill.
\newblock Off-policy policy gradient with state distribution correction.
\newblock \emph{arXiv preprint arXiv:1904.08473}, 2019.

\bibitem[Murphy et~al.(2001)Murphy, van~der Laan, and Robins]{murphy01marginal}
Susan~A. Murphy, Mark van~der Laan, and James~M. Robins.
\newblock Marginal mean models for dynamic regimes.
\newblock \emph{Journal of the American Statistical Association}, 96\penalty0
  (456):\penalty0 1410--1423, 2001.

\bibitem[Nachum et~al.(2019)Nachum, Chow, Dai, and Li]{nachum2019dualdice}
Ofir Nachum, Yinlam Chow, Bo~Dai, and Lihong Li.
\newblock Dualdice: Behavior-agnostic estimation of discounted stationary
  distribution corrections.
\newblock \emph{Neural Information Processing Systems (NeurIPS)}, 2019.

\bibitem[Puterman(2014)]{puterman2014markov}
Martin~L Puterman.
\newblock \emph{Markov Decision Processes.: Discrete Stochastic Dynamic
  Programming}.
\newblock John Wiley \& Sons, 2014.

\bibitem[Strehl et~al.(2010)Strehl, Langford, Li, and Kakade]{strehl11learning}
Alexander~L. Strehl, John Langford, Lihong Li, and Sham~M. Kakade.
\newblock Learning from logged implicit exploration data.
\newblock In \emph{Advances in Neural Information Processing Systems 23
  (NIPS-10)}, pp.\  2217--2225, 2010.

\bibitem[Sutton \& Barto(1998)Sutton and Barto]{sutton98beinforcement}
Richard~S. Sutton and Andrew~G. Barto.
\newblock \emph{Reinforcement Learning: An Introduction}.
\newblock MIT Press, Cambridge, MA, March 1998.
\newblock ISBN 0-262-19398-1.

\bibitem[Thomas \& Brunskill(2016)Thomas and Brunskill]{thomas16data}
Philip~S. Thomas and Emma Brunskill.
\newblock Data-efficient off-policy policy evaluation for reinforcement
  learning.
\newblock In \emph{Proceedings of the 33rd International Conference on Machine
  Learning (ICML)}, pp.\  2139--2148, 2016.

\bibitem[Thomas et~al.(2017)Thomas, Theocharous, Ghavamzadeh, Durugkar, and
  Brunskill]{thomas17predictive}
Philip~S. Thomas, Georgios Theocharous, Mohammad Ghavamzadeh, Ishan Durugkar,
  and Emma Brunskill.
\newblock Predictive off-policy policy evaluation for nonstationary decision
  problems, with applications to digital marketing.
\newblock In \emph{Proceedings of the 31st AAAI Conference on Artificial
  Intelligence (AAAI)}, pp.\  4740--4745, 2017.

\bibitem[Wang et~al.(2017)Wang, Agarwal, and Dud\'{i}k]{wang17optimal}
Yu-Xiang Wang, Alekh Agarwal, and Miroslav Dud\'{i}k.
\newblock Optimal and adaptive off-policy evaluation in contextual bandits.
\newblock In \emph{Proceedings of the 34th International Conference on Machine
  Learning (ICML)}, pp.\  3589--3597, 2017.

\bibitem[Xie et~al.(2019)Xie, Ma, and Wang]{xie2019optimal}
Tengyang Xie, Yifei Ma, and Yu-Xiang Wang.
\newblock Optimal off-policy evaluation for reinforcement learning with
  marginalized importance sampling.
\newblock \emph{Neural Information Processing Systems (NeurIPS)}, 2019.

\end{thebibliography}
\bibliographystyle{iclr2020_conference}

\clearpage

\onecolumn
\appendix
\begin{center}
\Large
\textbf{Appendix}
\end{center}

\section{Proof}
\subsection{Transition Operator for Bellman Equation}
\label{section:operator}
For simplicity, we define the following two operators thorough our proofs to simplify our notations.
\begin{mydef}
Given a policy $\pi$ and the unknown environment transition $\T$, we define $\Tpi$ and $\Ppi$ over any function $f:\Sset\to \R$ as
\begin{align*}
    & \left(\Tpi f\right) (s') = \sum_{s,a} \T(s'|s,a) \pi(a|s) f(s)\\
    & \left(\Ppi f\right) (s) = \sum_{s',a} \T(s'|s,a) \pi(a|s) f(s')
\end{align*}
\label{def:operator}
\end{mydef}

Using these operator notations, we can rewrite the above two recursive equations as:
\begin{align*}
    &V^\pi = r^\pi + \gamma \Ppi V^\pi, \\
    &d_\pi = (1-\gamma) \mu_0 + \gamma \Tpi d_\pi,
\end{align*}
where $r^\pi(s) = \E_{a\sim \pi(\cdot|s)}[r(s,a)]$.

These transition operators have the following nice adjoint property.
\begin{lem}
For two function $f$ and $g$, if the following summation is finite, we will have
\begin{equation}
    \sum_{s} \left(\Ppi f\right)(s) g(s) = \sum_{s} f(s) \left(\Tpi g\right)(s).
\end{equation}
\label{lem:adjoint}
\end{lem}
\begin{proof}
\begin{align*}
    \sum_{s} \left(\Ppi f\right)(s) g(s) =& \sum_{s} \left(\sum_{s',a} \T(s'|s,a) \pi(a|s) f(s')\right) g(s) \\
    =& \sum_{s'} f(s') \left(\sum_{s,a} \T(s'|s,a) \pi(a|s) g(s)\right) \\
    =& \sum_{s'} f(s') \left(\Tpi g\right)(s')
\end{align*}
\end{proof}

Using this property we can actually using Bellman Equations to re-derive the two different way to get $\Rpi$.
\begin{align*}
    \Rpi =& \lim_{T\to \infty} \E_{\tau^{(i)} \sim \pi}[\frac{\sum_{t=0}^T \gamma^t r_t}{\sum_{t=0}^T \gamma^t}] \\
    =& \sum_{s} V^\pi(s) (1-\gamma) \mu_0(s) \\
    =& \sum_{s} V^\pi(s) \left(\mathcal{I} - \gamma \Tpi \right)d_\pi(s) \\
    =& \sum_{s} \left(\mathcal{I} - \gamma \Ppi \right)V^\pi(s) d_\pi(s) \\
    =& \sum_{s} r^\pi(s)d_\pi(s).
\end{align*}

\subsection{Proof of Theorem \ref{thm:biased}}
\begin{proof}
Using the property of the operator, we can rewrite $(1-\gamma) \mu_0(s)$ using Bellman equation as $d_\pi - \gamma \Tpi d_\pi$, thus we have
\begin{align*}
    \Rvali[\vhat] =& (1-\gamma) \sum_{s} \vhat(s) \mu_0(s) \\
    =& \sum_s \vhat(s) \left(d_\pi - \gamma \Tpi d_\pi \right)(s) \\
    =& \sum_s \left(I - \gamma \Ppi\right)\vhat (s) d_\pi(s).
\end{align*}
and similarly if we break $r^\pi$ as $(I-\gamma \Ppi)V^\pi$, for $\Rsisi[\what]$ we have:
\begin{equation*}
    \Rsisi[\what] = \sum_s \left(I - \gamma \Ppi\right)V^\pi (s) d_{\what}(s),
\end{equation*}
where $d_{\what} = d_{\pi_0} \what$ for short.

Compare with $\Rpi = \sum_s \left(I - \gamma \Ppi\right)V^\pi (s) d_\pi(s)$, we can see the main difference between $\Rsisi$ and $\Rvali$ with $\Rpi$ are they replace $d_\pi$ and $d_{\what}$ and $V^\pi$ with $\vhat$ respectively.
If we add them together and minus the connection estimator, we have
we will have:
\begin{align*}
    \Rdri[\vhat, \what] =& \Rsisi[\what] + \Rvali[\vhat] - \Rconni[\vhat, \what] - \Rpi \\
    =& \sum_s \left(\left(I - \gamma \Ppi\right)V^\pi (s) d_{\what}(s) + \left(I - \gamma \Ppi\right)\vhat (s) d_\pi(s) - \left(I - \gamma \Ppi\right)\vhat (s) d_{\what}(s) - \left(I - \gamma \Ppi\right)V^\pi (s) d_\pi(s)\right) \\
    =& \sum_s \left(I - \gamma\Ppi \right) (V^\pi - \vhat)(s) (d_{\what} - d_\pi)(s)
\end{align*}
where $\left(I - \gamma\Ppi \right) (V^\pi - \vhat) = \left(I - \gamma\Ppi \right) V^\pi - \left(I - \gamma\Ppi \right) \vhat = r^\pi - \left(I - \gamma\Ppi \right) \vhat$.
\end{proof}

\subsection{More discussions on the Variance in Theorem \ref{thm:variance}}
\begin{thm}
Let $\var\left [ \Rres[\vhat, \what] \right ]$ be defined in \ref{thm:variance}, and suppose we can uniformly draw samples from $d_{\pi_0}$ to form empirical $\E_{\widehat{d_{\pi_0}}}$ (in practice we can draw sample $s_t$ depends on its discounted factor $\gamma_t$). 
Then we can further break it into two terms.
\begin{equation}
    \var_{\mathcal D_{\pi_0}}\left [ \Rres[\vhat, \what] \right ] = \frac{1}{n}(\var\left[\what(s) \varepsilon_{\vhat}(s) \right] + \E\left[\what(s)^2 \left(\delta_1(s,a) + \gamma \delta_2(s,a,s')\right)^2\right]),
\end{equation}
where $\varepsilon_{\vhat}(s) = \vhat(s) - r^\pi(s) - \gamma \Ppi \vhat(s')$ is the Bellman residual, $\delta_1(s,a) = \frac{\pi(a|s)}{\pi_0(a|s)} r(s,a) - r^\pi(s)$ is the randomness for action and $\delta_2(s,a,s') = \frac{\pi(a|s)}{\pi_0(a|s)} \vhat(s') - \Ppi \vhat(s)$ is the randomness for transition operator over function $\vhat$.
Both $\delta_1$ and $\delta_2$ is zero mean if we condition over $s$.

Compared with $\var\left[\Rsis[\what]\right]$ we have:
\begin{equation}
    \var\left[\Rsis[\what]\right] = \frac{1}{n}(\var\left[\what(s) r^\pi(s) \right] + \E\left[\what(s)^2 \delta_1(s,a)^2\right])
\end{equation}
\end{thm}
\begin{proof}
$\Rres[\vhat, \what]$ can be written as
\begin{align*}
    \Rres[\vhat, \what] = \frac{1}{n} \sum \what(s)\left(\betar(a|s) (r + \gamma \vhat(s')) - \vhat(s)\right),
\end{align*}
where $\betar(a|s)$ is short for $\frac{\pi(a|s)}{\pi_0(a|s)}$.
We can break $\betar(a|s)(r + \gamma \vhat(s')) - \vhat(s)$ into
\begin{align*}
    &\betar(a|s)(r(s,a) + \gamma \vhat(s')) - \vhat(s)\\
    =& (-\vhat(s) + r^\pi(s) + \gamma \Ppi \vhat (s)) + \underbrace{(\betar(a|s) r(s,a) - r^\pi(s))}_{\delta_1} + \gamma \underbrace{(\betar(a|s) \vhat(s') - \Ppi \vhat(s))}_{\delta_2} \\
    =& -\varepsilon_{\vhat}(s) + \delta_1(s,a) + \gamma \delta_2(s,a,s').
\end{align*}
where $\varepsilon_{\vhat} = \vhat - r^\pi - \Ppi \vhat$ is the Bellman residual and the if we condition over $s$ we have the expectations for $\delta_1$ and $\delta_2$ are 0.
Also notice that if we condition over $s$ then $\varepsilon_{\vhat}(s)$ become a constant thus it is independent to $\delta_1$ and $\delta_2$.
Thus we have:
\begin{align*}
    &\var\left[\what(s) \left(\betar(a|s)(r + \gamma \vhat(s')) - \vhat(s)\right)\right] \\
    =& \var\left[\what(s) \left( -\varepsilon_{\vhat}(s) + \delta_1(s,a) + \gamma \delta_2(s,a,s') \right)\right] \\
    =& \var\left[\what(s) \varepsilon_{\vhat}(s) \right] + \E\left[\what(s)^2 \left(\delta_1(s,a) + \gamma \delta_2(s,a,s')\right)^2\right]
\end{align*}

Therefore we have:
\begin{equation*}
    \var\left[\Rdr[\vhat, \what]\right] = \frac{(1-\gamma)^2}{n_0} \var[\vhat(s_0)] + \frac{1}{n}(\var\left[\what(s) \varepsilon_{\vhat}(s) \right] + \E\left[\what(s)^2 \left(\delta_1(s,a) + \gamma \delta_2(s,a,s')\right)^2\right]).
\end{equation*}

For $\var\left[\Rsis[\what]\right]$ we have:
\begin{align*}
    \var\left[\Rsis[\what]\right] =& \frac{1}{n}\var\left[\what(s) \betar(a|s) r(s,a)\right] \\
    =& \frac{1}{n} \var\left[\what(s) r^\pi(s) + \what(s) \delta_1(s,a)\right] \\
    =& \frac{1}{n} (\var\left[\what(s) r^\pi(s)\right] + \E\left[\what^2(s) \delta_1(s,a)^2\right]). 
\end{align*}
\end{proof}

From the theorem we can see that the variance of residual comes from two parts, the majority part relies on the variance of $|\varepsilon_{\vhat}(s)|$ is usually much smaller than $r^\pi$ as the majority variance of state visitation importance sampling.

\subsection{Proof of Theorem \ref{thm:primal-dual}}
\begin{proof}
The Lagrangian can be written as:
\begin{align*}
    L(V,\rho) =& (1-\gamma) \sum_{s} \mu_0(s) V(s) - \sum_{s} \rho(s) \left(V(s) - r^\pi(s) - \gamma \Ppi V(s)\right) \\
    =&  \underbrace{\sum_{s} (1-\gamma)\mu_0(s) V(s)}_{=\Rvali[V]} - \underbrace{\sum_s \rho(s) \left(I-\gamma \Ppi\right) V(s)}_{=\Rconni[V,w_{\rho/\pi_0}]} + \underbrace{\sum_s \rho(s) r^\pi(s)}_{=\Rsisi[w_{\rho/\pi_0}]} \\
    =&  \sum_{s} (1-\gamma)\mu_0(s) V(s) - \sum_s \left(I-\gamma \Tpi\right)\rho(s)  V(s) + \sum_s \rho(s) r^\pi(s) \\
    =& \sum_s \left((1-\gamma)\mu_0(s) - (I-\gamma \Tpi)\rho(s)\right) V(s) + \sum_s \rho(s) r^\pi(s).
\end{align*}
We can see that the Lagrangian $L(V,\rho)$ is actually our doubly robust estimator $\Rdri[V, w_{\rho/\pi_0}]$.

From the last equation we can derive our dual as:
\begin{align*}
    \max_{\rho\geq 0} ~~~&\sum_s \rho(s) r^\pi(s) \\
    \text{s.t. }~~& \rho(s) = (1-\gamma) \mu_0(s) + \gamma \Tpi \rho(s),~\forall s.
\end{align*}
\end{proof}

\clearpage
\section{Doubly Robust Estimator for Average Case}
\label{sec:avg_prime}
\subsection{Primal Dual Framework}
We start from primal dual framework to get our doubly robust estimator similar to section \ref{sec:duality}.
To estimate the average reward of a given policy $\pi$, we can consider solve the following linear programming:
\begin{align}
    \max_{\rho \geq 0}& \sum_{s}\rho(s)r^{\pi}(s) \notag \\ 
    \rm{s.t.} & ~~\sum_{s} \rho(s) = 1, ~~~\rho(s) = \Tpi\rho(s),~\forall s, \label{equ:avg_ratio}
\end{align}
where $\rho(s)$ is the stationary distribution of states under $\Ppi$, and the objective is the average reward given $\pi$.

Consider the Lagrangian of above linear programming:
\begin{align}
    L(V, \rho, \bar{v}) &= \sum_{s}\rho(s)r^{\pi}(s) - \sum_{s}V(s)(\rho(s) - \Tpi\rho(s)) - \bar{v}(\sum_{s}\rho(s) - 1) \notag \\
    & = \sum_{s}\rho(s)(r^{\pi}(s) - V(s) - \Ppi V(s) - \bar{v}) + \bar{v}. \label{equ:avg_lag}
\end{align}

From Equation \eqref{equ:avg_lag} we can get the dual formula as:
\begin{align}
    \min_{V, \bar{v}}&~~\bar{v} \notag \\ 
    \rm{s.t.} & ~~~\bar{v} + V(s) - \Ppi V(s) - r^{\pi}(s) \geq 0, ~~\forall s,
\end{align}
where $V(s)$ is the value function and $\bar{v}$ is the average reward we want to optimize.

Notice that in average case, $V^\pi(s)$ can be viewed as the fixed-point solution to the following Bellman equation:
$$
V^\pi(s) - \E_{s^\prime, a | s \sim d_{\pi}}[V^\pi(s^\prime)] = \E_{a|s  \sim \pi}[r(s, a) - \bar{v}].
$$
Note that this explains the constraint and only if we pick $\bar v = R^\pi$, we can find a $V$ to guarantee the constraint $\bar{v} + V(s) - \Ppi V(s) - r^{\pi}(s) \geq 0$ holds true.

\subsection{Doubly Robust Estimator}
We want to build the doubly robust estimator via the lagrangian.
However, the Lagrangian consist of three term $\rho, V$ and $\bar v$.
It is counter-intuitive if we already given an estimator of $\bar v \approx R^\pi$ into our estimator.

A better way to solve this problem is to remove the constraint $\sum \rho(s) = 1$, but we divide it as an self-normalization.
Then our Lagrangian becomes
$$
L(V,\rho) = \frac{\sum_{s}\rho(s)r^{\pi}(s) - \sum_{s}V(s)(\rho(s) - \Tpi\rho(s))}{\sum \rho(s)}.
$$
which can be utilized to define the doubly robust estimator for average reward.
\begin{mydef} Given a learned value function ${\hat{V}}(s)$ for policy $\pi$ and an estimated density ratio $\hat{w}(s)$ for $\rnd(s)$, we define
$$
    \Rdr[\vhat, \what] := \frac{\sum_{s,a,r,s'\in \D}\what(s) \left(\betar(a|s) (r+ \vhat(s')) - \vhat(s) \right)}{\sum_{s\in \D}\what(s)},
$$
where $\betar(a|s) = \frac{\pi(a|s)}{\pi_0(a|s)}$.
\end{mydef}

Similarly to Theorem \ref{thm:biased} we will have our double robustness for our average doubly robust estimator:
\begin{thm}
Suppose we have infinite samples and we can get
$$
\Rdri[\vhat, \what] = \frac{\E_{s\sim d_{\pi_0}}\left[\what(s)\left( r_\pi(s) - \vhat(s) + \Ppi \vhat(s) \right)\right]}{\E_{s\sim d_{\pi_0}}\left[\what(s)\right]}.
$$
Then we have
\begin{equation}
    \Rdri[\vhat, \what] - \Rpi =
    \E_{s\sim d_{\pi_0}}\left[\varepsilon_{\what}(s)\varepsilon_{\vhat}(s)\right],  
\end{equation}
where $\varepsilon_{\vhat} $ and $\varepsilon_{\what}$ are errors of $\vhat$ and $\what$, respective, defined as follows 
\begin{align*} 
\varepsilon_{\what} =\frac{\what(s)}{\E_{s\sim d_{\pi_0}}\left[\what(s)\right]} - \frac{d_\pi(s)}{d_{\pi_0}(s)}, &&
\varepsilon_{\vhat}(s) = r_\pi - \vhat + \Ppi \vhat - \Rpi. 
\end{align*}
\end{thm}
\begin{proof}
A key observation is that
\begin{align*}
    \E_{s\sim d_{\pi_0}}[\rnd(s) \varepsilon_{\vhat}(s)] =& \E_{s\sim d_\pi}[r^\pi(s) - \vhat(s) + \Ppi \what(s) - R^\pi] \\
    =& \left(\E_{s\sim d_\pi}[r^\pi(s)]-\Rpi\right)  + \E_{s\sim d_\pi}[- \vhat(s) + \Ppi \what(s)] \\
    =& 0
\end{align*}
Thus we have
\begin{align*}
    \Rdri[\vhat, \what] - \Rpi =& \E_{s\sim d_{\pi_0}}\left[\frac{\what(s)}{\E_{s\sim d_{\pi_0}}\left[\what(s)\right]} \left(R^\pi + \varepsilon_{\vhat}(s)\right)\right] - \Rpi\\
    =& \E_{s\sim d_{\pi_0}}\left[\frac{\what(s)}{\E_{s\sim d_{\pi_0}}\left[\what(s)\right]}\varepsilon_{\vhat}(s)\right] \\
    =& \E_{s\sim d_{\pi_0}}\left[\frac{\what(s)}{\E_{s\sim d_{\pi_0}}\left[\what(s)\right]}\varepsilon_{\vhat}(s)\right] - \E_{s\sim d_{\pi_0}}[\rnd(s) \varepsilon_{\vhat}(s)] \\
    =& \E_{s\sim d_{\pi_0}}\left[\varepsilon_{\what}(s)\varepsilon_{\vhat}(s)\right].
\end{align*}
\end{proof}

Similar to discounted case we have $\Rdri[\vhat, \what] = R^{\pi}$ if either $\what$ or $\vhat$ is accurate.
\clearpage
\section{Experimental Details}
\label{sec:exp}
\subsection{Tabular Case: Taxi}
\label{sec:taxi}
\paragraph{Behavior and Target Policies Choosing}
We use an on-policy Q-learning to get a sequence of policy $\pi_0, \pi_1,...,\pi_{19}$ as data size increases.
We pick the last policy $\pi_{19}$ (almost optimum) as our target policy and $\pi_{18}$ as our behavior policy to guarantee that those policies are not far away.
We set our discounted factor $\gamma = 0.99$.

\paragraph{Train $\vhat$ and $\widehat{\rho}$}
Separate from testing, we use a set of independent sample to first 
train a value function $\vhat$ and a density function $\widehat{\rho}$.
Both $\vhat$ and $\widehat{\rho}$ have bias due to finite sample approximation.

\newcommand{\qhat}{\widehat{Q}^\pi}
\newcommand{\rhohat}{\widehat{\rho}}
\newcommand{\muhat}{\widehat{\mu}}
For how to train $\vhat$ and $\widehat{\rho}$,
we choose to use Monte Carlo method to estimate $\vhat$ and $\widehat{\rho}$.
We first use the finite samples to get an estimated model $\widehat{T}(s'|s,a)$ and rewards function $\widehat{r}(s,a)$ and $\widehat{d}_0$.
Then we solve the following linear equation (by iteration like power method, which is actually Monte Carlo):

\begin{align*}
    \vhat(s) &= \sum_a \pi(a|s) \qhat(s,a), \\
    \qhat(s,a) &= \widehat{r}(s,a) + \gamma \sum_{s'} \widehat{T}(s'|s,a) \vhat(s'), \\
    \muhat(s,a) &= \rhohat(s) \pi(a|s), \\
    \rhohat(s) &= (1-\gamma) \widehat{d}_0(s) + \gamma \sum_{s,a}\widehat{T}(s'|s,a) \muhat(s,a).
\end{align*}

\paragraph{Estimate $\Rpi$ Using $\vhat$ and $\rhohat$}
We put $\vhat$ and $\widehat{\rho}$ into the Lagrangian as equation \eqref{eqn:lagrangian} as our doubly robust estimator.
For those states we haven't visited, we set $\vhat(s)$ and $\widehat{\rho}(s)$ as $0$ and we self-normalized the $\widehat{\rho}$ to get a fair estimation.

\subsection{Continuous States Off-Policy Evaluation}
\paragraph{Evaluation Environments}
We evaluate our method on two infinite horizon environments: Puck-Mountain and InvertedPendulum. 

Puck-Mountain is an environment similar to Mountain-car, except that the goal of the task is to push the puck as high as possible in a local valley whose initial position is at the bottom of the valley.
If the ball reaches the top sides of the valley, 
it will hit a roof and changes the speed to its opposite direction with half of its original speeds. The rewards was determined by the current velocity and height of the puck.

InvertedPendulum is a pendulum that has its center of mass above its pivot point. It is unstable and without additional help will fall over.
We train a near optimal policy that can make the pendulum balance for a long horizon.
For both behavior and target policies, we assume they are good enough to keep the pendulum balance and will never fall down until they reach the maximum timesteps.
We use the implementation from OpenAI Gym \citep{brockman2016openai} and changing the dynamic by adding some additional zero mean Gaussian noise to the transition dynamic.

\label{sec:exp_continuous}
\paragraph{Behavior and Target Policies Learning}
We use the open source implementation\footnote{https://github.com/openai/baselines} of deep Q-learning to train a $32 \times 32$ MLP parameterized Q-function to converge.
We then use the softmax policy of learned the Q-function as the target policy $\pi$, which has a default temperature $\tau = 1$.
For the behavior policy $\pi_{0}$, we set a relative large temperature which encourages exploration.
We set the temperature of the behavior policy $\tau_0 = 1.88$ for Puck-Mountain and $\tau_0 = 1.50$ for InvertedPendulum respectively.
\paragraph{Training of density ratio $\hat{w}(s)$ and value function $\hat{V}(s)$}
We use a seperate training dataset with 200 trajectories whose horizon length is 1000 to learn the density ratio $\hat{w}(s)$ and the value function $\hat{V}(s)$.
For the training of density ratio, we adapt the algorithm 2 in \citet{liu2018breaking} to train a neural network parameterized $w_{\theta}(s)$. 
Instead of taking the test function $f(s)$ into an RKHS $\mathcal{H}_{\mathcal{K}}$, we parameterize the test function $f(s) = f_{\beta}(s)$ to be a neural network with parameter $\beta$, and perform minimax optimization over parametr $\theta$ and $\beta$. A detail description can be found in Algorithm \ref{alg:discounted}.

Similarly, for the training of value function, we use primal-dual based optimization methods \citep{dai2017sbeed,feng2019kernel} to minimize the bellman residual:
$$
    \min_{\phi}\max_{f_{\beta}\in \mathcal{F}}\frac{1}{|\mathcal{M}|}\sum_{i \in \mathcal{M}}\left(\left(V_{\phi}(s_i) - \frac{\pi(a_i| s_i)}{\pi_{0}(a_i|s_i)}(r_i + \gamma V_{\phi}(s_i^{\prime}))\right)f_{\beta}(s_i) - \frac{1}{2}f_{\beta}(s_i)^2\right),
$$
where $V_{\phi}(s)$ is the parameterized value function and $f_{\beta}(s)$ is the test function.
We also perform minimax over parameter $\phi$ and $\beta$. A detail description can be found in Algorithm \ref{alg:dis_val}.

For the network structures, we use $32\times32$ feed forward neural networks to parameterize value function $V_{\phi}$ and density ratio $w_\theta(s)$, and we use one hidden neural network with 10 units to parameterize the test function
$f_{\beta}(s)$. We use Adam Optimizer for all our experiments.

\paragraph{Estimate $R^{\pi}$ using $\vhat$ and $\what$} Given data samples from the policy $\pi_{0}$, 
We can directly use $\Rdr$ in equation \eqref{eqn:doubly_robust} to estimate $\Rpi$.

\begin{algorithm}[t] 
\caption{Optimization of density ratio $\what$}  
\label{alg:discounted}
\begin{algorithmic} 
\STATE {\bf Input}: Transition data $\mathcal D = \{s_i, a_i, s'_i, r_i\}_{i=1}^n$ from the behavior policy  $\pi_0$; a target policy $\pi$ for which we want to estimate the expected reward. Discount factor $\gamma\in(0,1)$, starting state $\D_0 = \{s^{(0)}_j\}_{j=1}^m$ from initial distribution.
\STATE {\bf Initial} the density ratio $w(s) = w_\theta(s)$ to be a neural network parameterized by $\theta$, $f(s) = f_{\beta}(s)$ to be a neural network parameterized by $\beta$. We need to ensure that the final layer of $\theta$ is a softmax layer.
\FOR{iteration = 1,2,...,T} 
\STATE Randomly choose a batch $\mathcal M\subseteq \{1,\ldots,n\}$ uniformly from the transition data $\D$ and a batch $\mathcal M_0 \subseteq \{1,\ldots,m\}$ uniformly from start states $\D_0$.
\FOR{iteration = 1,2,..., K}
\STATE {\bf Update} the parameter $\beta$ by $\beta \gets \beta + \epsilon_\beta \nabla_{\beta} \hat{L}(w_\theta, \f_\beta)$, where 
$$
\hat{L}(w_\theta, f_\beta) = \frac{1}{|\mathcal M|} \sum_{i \in \mathcal M} \left((w(s_i') - \gamma w(s_i)\frac{\pi(a_i|s_i)}{\pi_0(a_i|s_i)} - \frac{1}{2}f(s_i'))f(s'_i) \right) - (1-\gamma) \frac{1}{|\mathcal M_0|} \sum_{j\in \mathcal M_0} f(s_j) 
$$
\ENDFOR
\STATE {\bf Update} the  parameter $\theta$ by $\theta \gets \theta - \epsilon_{\theta}\nabla_\theta \hat{L}(w_\theta, f_\beta)$.
\ENDFOR
\STATE {\bf Output}: the density ratio $\what = w_{\theta}$.
\end{algorithmic} 
\end{algorithm}

\begin{algorithm}[t] 
\caption{Optimization of value function $\vhat$}  
\label{alg:dis_val}
\begin{algorithmic} 
\STATE {\bf Input}: Transition data $\mathcal D = \{s_i, a_i, s'_i, r_i\}_{i=1}^n$ from the behavior policy  $\pi_0$; a target policy $\pi$ for which we want to estimate the expected reward. Discount factor $\gamma\in(0,1)$. 
\STATE {\bf Initial} the value function $V(s) = V_\phi(s)$ to be a neural network parameterized by $\phi$, $f(s) = f_{\beta}(s)$ to be a neural network parameterized by $\beta$.
\STATE 
\FOR{iteration = 1,2,...,T} 
\STATE Randomly choose a batch $\mathcal M\subseteq \{1,\ldots,n\}$ uniformly from the transition data $\D$.
\FOR{iteration = 1,2,..., K}
\STATE {\bf Update} the parameter $\beta$ by $\beta \gets \beta + \epsilon_\beta \nabla_{\beta} \hat{L}(V_\phi, \f_\beta)$, where 
$$
\hat{L}(V_\phi, f_\beta) = \frac{1}{|\mathcal{M}|}\sum_{i \in \mathcal{M}}\left(\left(V_{\phi}(s_i) - \frac{\pi(a_i| s_i)}{\pi_{0}(a_i|s_i)}(r_i + \gamma V_{\phi}(s_i^{\prime}))\right)f_{\beta}(s_i) - \frac{1}{2}f_{\beta}(s_i)^2\right)
$$
\ENDFOR
\STATE {\bf Update} the  parameter $\phi$ by $\phi \gets \phi - \epsilon_{\phi}\nabla_\phi \hat{L}(V_\phi, f_\beta)$.
\ENDFOR
\STATE {\bf Output}: the density ratio $\vhat = V_{\phi}$.
\end{algorithmic} 
\end{algorithm}

\end{document}